\documentclass[10pt,twocolumn,letterpaper]{article}
\usepackage{nunchux}

\newcommand{\TeaserRaise}{-20pt}

\usepackage{cuted}

\usepackage{caption}

\usepackage{microtype}

\newcommand{\myparagraph}[1]{\par\vspace{.5em}\noindent\textbf{#1}\ \ignorespaces}

\usepackage{amsmath}
\usepackage{amsfonts}
\usepackage{amssymb}
\usepackage{amsthm}
\usepackage{bm}
\usepackage{booktabs}
\usepackage{graphicx}
\usepackage{multirow}
\usepackage{tabularx}
\usepackage{makecell}
\newcolumntype{Y}{>{\centering\arraybackslash}X}
\usepackage{xspace}
\usepackage{xcolor}
\usepackage{colortbl}
\usepackage{enumitem}
\usepackage{tikz}
\usetikzlibrary{arrows.meta,positioning,decorations.pathreplacing,calc,matrix,fit,patterns}
\usepackage{listings}
\newtheorem{proposition}{Proposition}[section]

\newcommand{\method}{VC-Attention\xspace}
\newcommand{\QK}{\mathbf{Q}\mathbf{K}^{\top}}
\newcommand{\PV}{\mathbf{PV}}

\definecolor{oursrow}{HTML}{F0EEF8}%
\newcommand{\ours}{\rowcolor{oursrow}}
\newcommand{\oursg}{\rowcolor{oursrow}\cellcolor{white}}

\newcommand{\na}{\textcolor{black!45}{--}}
\newcommand{\nanc}{NaN}
\newcommand{\GroupingStepFraction}{25\%\xspace}
\newcommand{\PermutationReuseSteps}{4\xspace}

\newcommand{\NumPrompts}{100\xspace}
\newcommand{\VBenchDimA}{subject consistency\xspace}
\newcommand{\VBenchDimB}{imaging quality\xspace}
\newcommand{\BlockSizeV}{128\xspace}
\newcommand{\ExpCastFidelityCost}{0.7--2.1~dB\xspace}
\newcommand{\HadamardPSNRGain}{0.3~dB\xspace}

\newcommand{\AttnSpeedupHTwoHundred}{1.50\texttimes\xspace}
\newcommand{\AttnSpeedupBTwoHundred}{1.59\texttimes\xspace}
\newcommand{\AttnSpeedupNaiveBTwoHundred}{1.24\texttimes\xspace}
\newcommand{\AttnSpeedupNaiveHTwoHundred}{1.37\texttimes\xspace}
\newcommand{\AttnSpeedupRTXPro}{3.67\texttimes\xspace}
\newcommand{\AttnSpeedupRTXFiftyNinety}{5.24\texttimes\xspace}%

\newcommand{\DenoiseSpeedupDatacenterRange}{1.25--1.26\texttimes\xspace}
\newcommand{\DenoiseSpeedupRTXPro}{1.85\texttimes\xspace}
\newcommand{\DenoiseSpeedupRTXFiftyNinety}{2.14\texttimes\xspace}%
\newcommand{\LockedClockBTwoHundred}{1350~MHz\xspace}
\newcommand{\LockedSpeedupBTwoHundred}{1.28--1.30\texttimes\xspace}
\newcommand{\LockedNaiveSpeedup}{0.91\texttimes\xspace}
\newcommand{\AttnQATTrainingFreeGap}{3.4--6.7~dB\xspace}
\newcommand{\HadamardPSNRGainAttnQAT}{0.5~dB\xspace}
\newcommand{\AttnQATDenoiseBThreeHundred}{1.11\texttimes\xspace}
\newcommand{\BThreeHundredGainAttnQAT}{1.3--4.1~dB\xspace}
\newcommand{\BThreeHundredGainDirectP}{3.3--6.0~dB\xspace}
\newcommand{\VBenchSpreadFourBit}{0.02\xspace}
\newcommand{\LPIPSGainEightBit}{13--29\%\xspace}
\newcommand{\SameByteShare}{79.6\%\xspace}
\newcommand{\MeasuredTV}{1.6\%\xspace}
\newcommand{\CastTV}{1.1\%\xspace}
\newcommand{\FourBitTV}{6.0\%\xspace}
\newcommand{\FourBitCastTV}{5.2\%\xspace}
\newcommand{\DirectPTVFast}{14.2\%\xspace}
\newcommand{\DirectPTVAccurate}{13.1\%\xspace}
\newcommand{\DirectPZeroMass}{5.7\%\xspace}
\newcommand{\SageConstVRows}{97--100\%\xspace}
\newcommand{\SageConstVShrink}{1--3\%\xspace}
\newcommand{\DenOursPSNR}{18.6\xspace}
\newcommand{\DenFPPSNR}{15.9\xspace}
\newcommand{\DenSagePSNR}{15.8\xspace}
\newcommand{\DenDrop}{2.7~dB\xspace}
\newcommand{\KMeansColdCall}{5.9~ms\xspace}
\newcommand{\KMeansWarmCall}{2.3~ms\xspace}
\newcommand{\GroupingOverhead}{3.5\%\xspace}
\newcommand{\GroupingOverheadEveryStep}{13.4\%\xspace}
\newcommand{\GroupingEveryStepCall}{61.99~ms\xspace}

\newcommand{\PrepFusionTotal}{7.65\texttimes\xspace}
\newcommand{\PrepGroupingShare}{44\%\xspace}
\newcommand{\PrepFusedCall}{5.3~ms\xspace}
\newcommand{\PrepUnfusedForward}{1609 ms\xspace}
\newcommand{\PrepFusedForward}{210 ms\xspace}
\newcommand{\BalancedGain}{13\%\xspace}
\newcommand{\VSeqEnergy}{about 17\%\xspace}
\newcommand{\VGaussEnergy}{about 18\%\xspace}
\newcommand{\SFUOverMMAFourBit}{3.8 times\xspace}
\newcommand{\ExpOnlyGain}{1.25--1.96\texttimes\xspace}
\newcommand{\VKMeansEnergy}{about 48\%\xspace}
\newcommand{\BalancedCost}{5.1\texttimes\xspace}

\newcommand{\AttnSpeedupBThreeHundredFourBit}{2.07\texttimes\xspace}

\newcommand{\Asset}[3][\textwidth]{\includegraphics[width=#1]{Figures/#2.pdf}}

\newcommand{\TeaserPSNRSage}{19.9~dB\xspace}
\newcommand{\TeaserPSNRVC}{20.2~dB\xspace}
\newcommand{\TeaserSpeedupSage}{{\fontencoding{T1}\selectfont\textless}1\texttimes\xspace}
\newcommand{\TeaserSpeedupVC}{1.59\texttimes\xspace}
\newcommand{\TeaserPrompt}{Prompt (abridged, full text in \cref{sec:appendix-qualitative}): \emph{A young parkour runner sprints down a marathon route, vaults a blue hydration table, and crashes through the finish-line tape.}}

\definecolor{nxblue}{HTML}{6451AC}
\definecolor{nxdark}{HTML}{43337C}
\definecolor{nxgrey}{HTML}{666666}
\definecolor{nxpale}{HTML}{F5F4FA}
\definecolor{nxsoft}{HTML}{A59DCD}
\definecolor{nxteal}{HTML}{1C5155}

\lstdefinestyle{nxcode}{%
  language=C,
  basicstyle=\ttfamily\scriptsize,
  commentstyle=\color{nxgrey},
  morekeywords={fp32,uint8,e4m3},
  keywordstyle=\color{nxteal},
  emph={exp2,to_e4m3,round,clip,view_as_e4m3,as_e4m3,to_nvfp4,ceil,max,as_f16,to_e2m1},
  emphstyle=\bfseries,
  columns=fullflexible, keepspaces=true, frame=none,
  aboveskip=3pt, belowskip=0pt, xleftmargin=0pt,
}
\newsavebox{\expcastconv}
\newsavebox{\expcastours}

\definecolor{verdictgood}{HTML}{7E9948}
\definecolor{verdictbad}{HTML}{A31F34}
\newcommand{\best}[1]{\textcolor{nxdark}{\textbf{#1}}}
\newcommand{\metric}[2]{\textbf{#1}: #2}

\providecommand{\TeaserRowGap}{6pt}
\providecommand{\TeaserLabelGap}{3pt}
\providecommand{\TeaserPromptGap}{8pt}
\providecommand{\TeaserBelow}{12pt}
\newcommand{\TeaserCols}{5}
\newlength{\TeaserFrameW}
\newcommand{\TeaserRowFA}{%
  \tframe{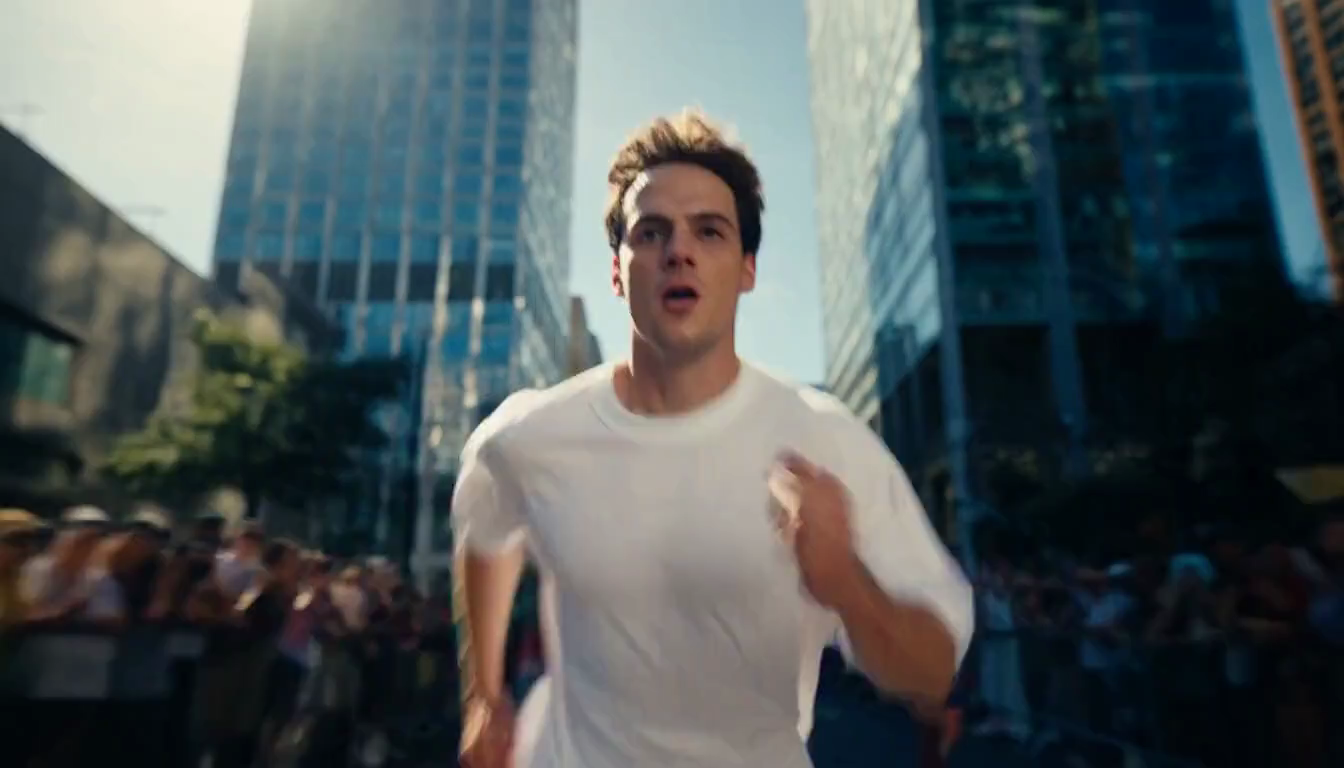} & \tframe{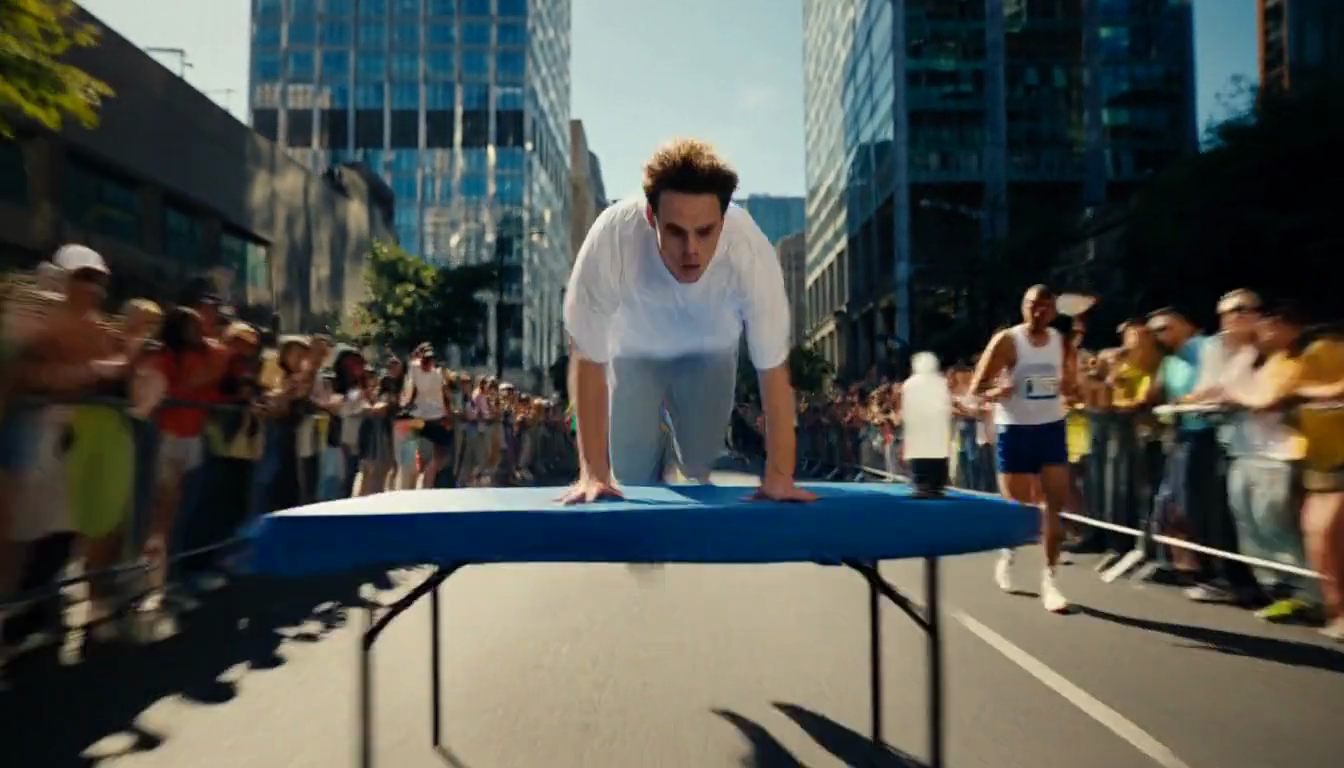} & \tframe{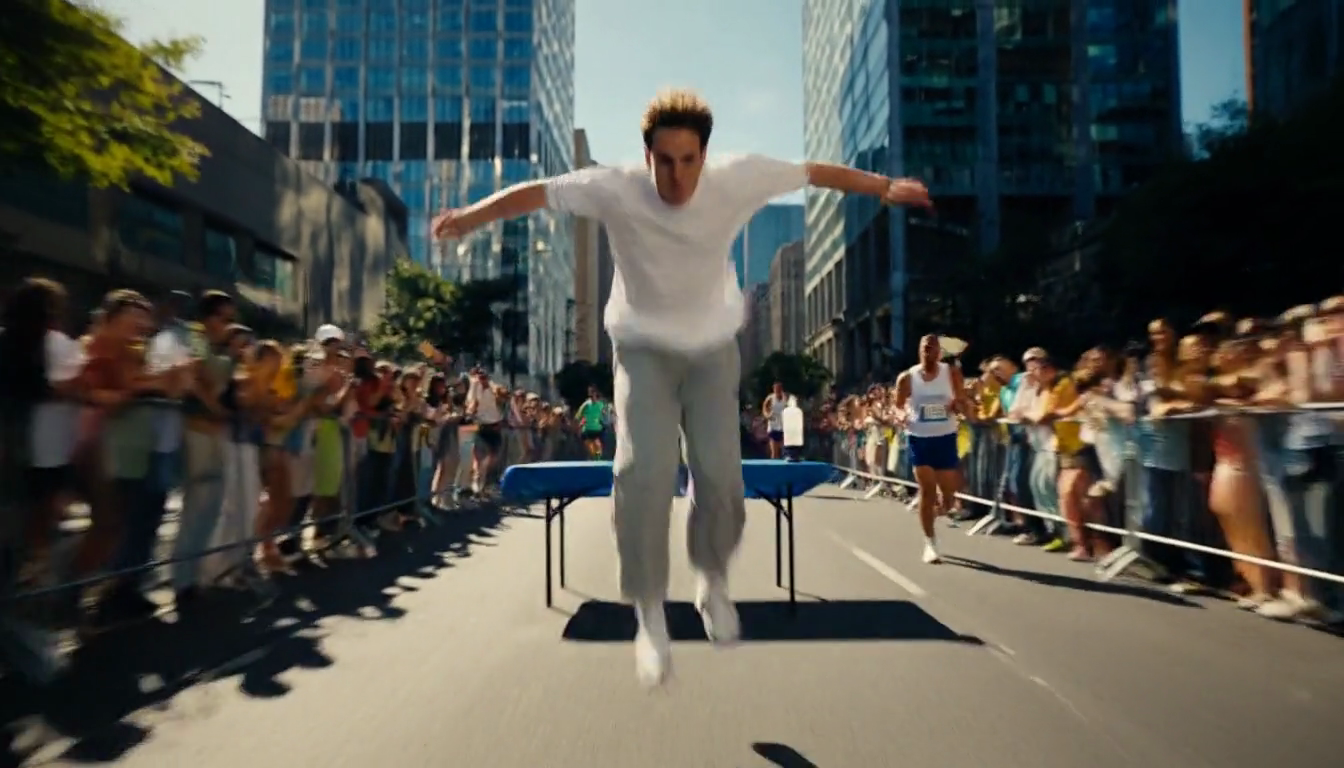} & \tframe{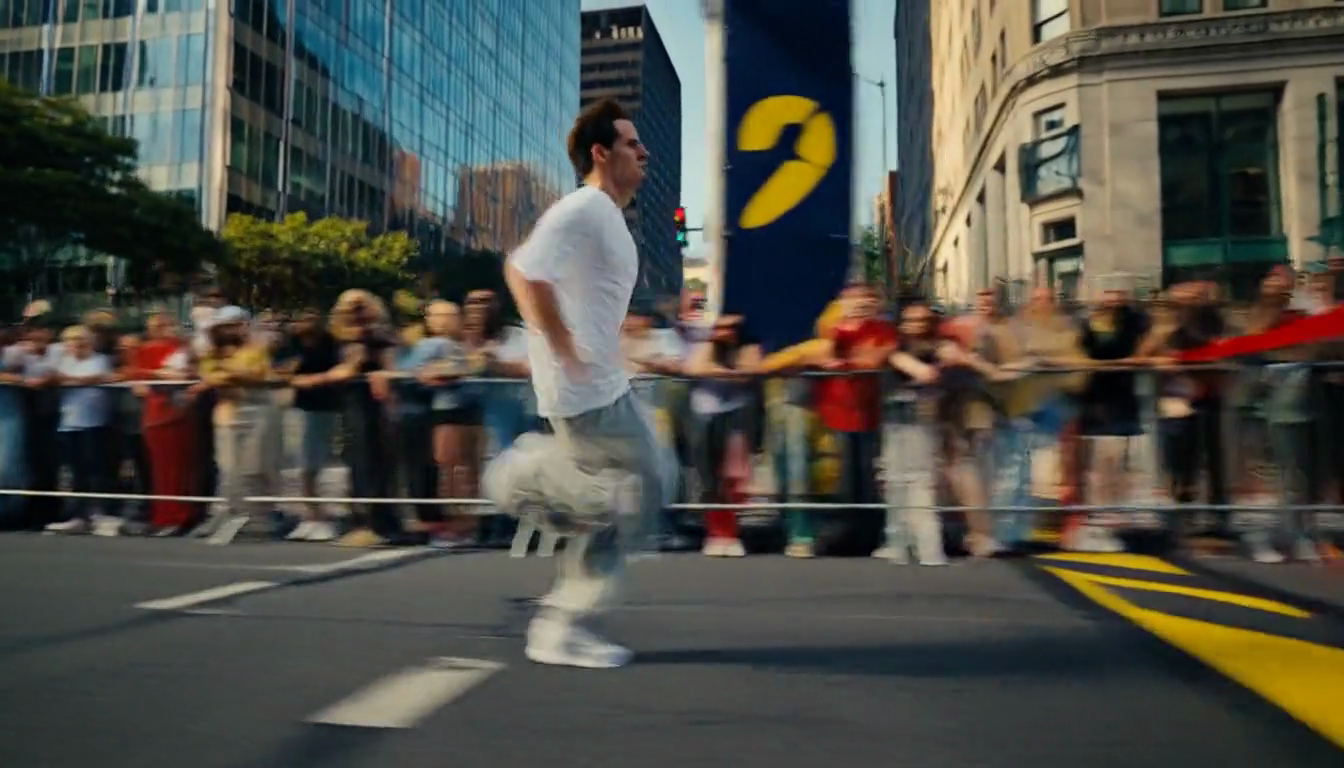} &
  \tframe{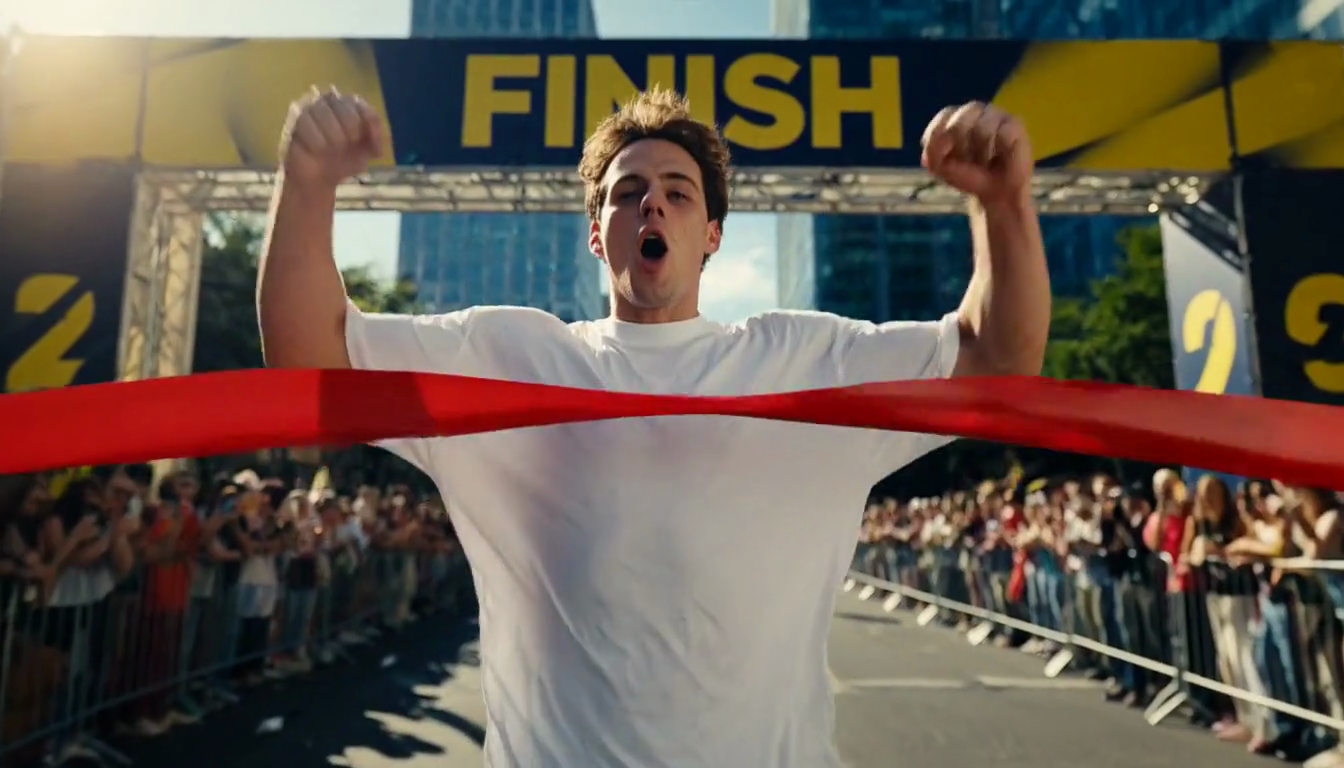}}
\newcommand{\TeaserRowSage}{%
  \tframe{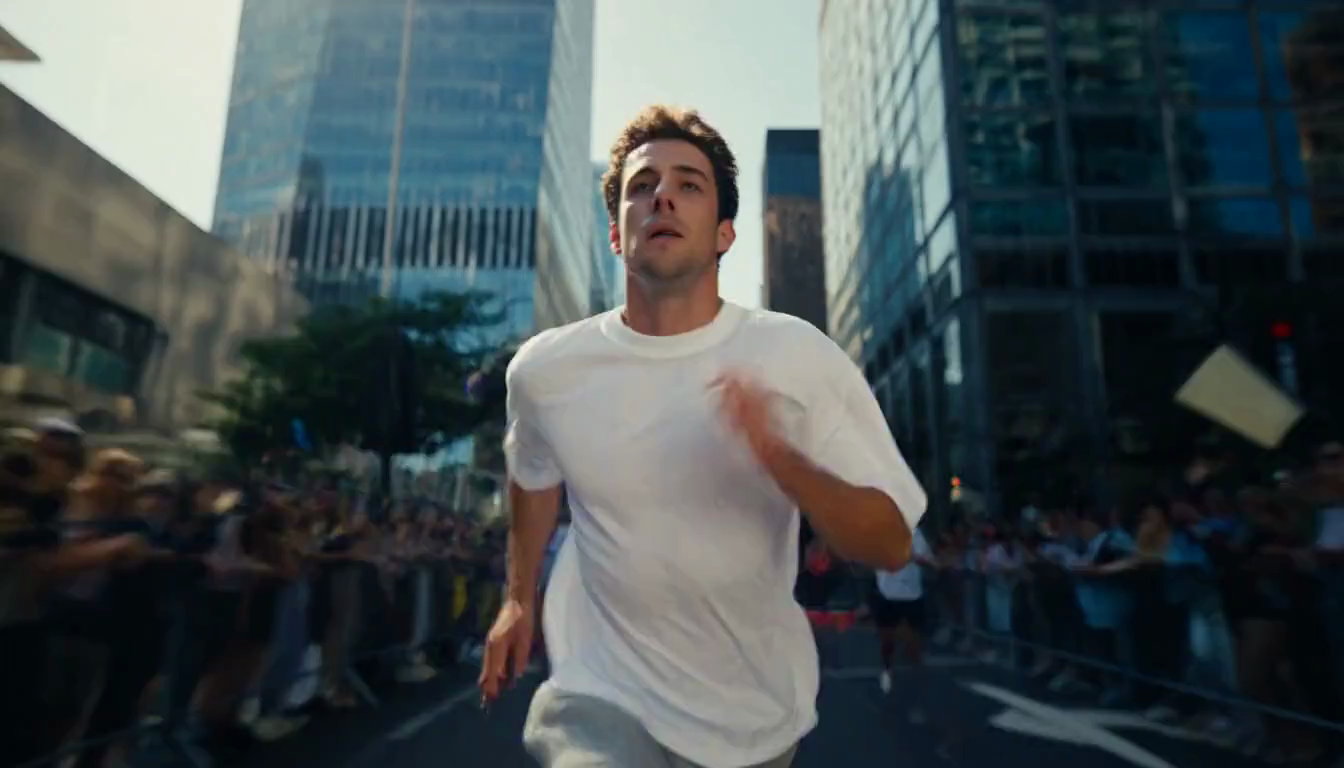} & \tframe{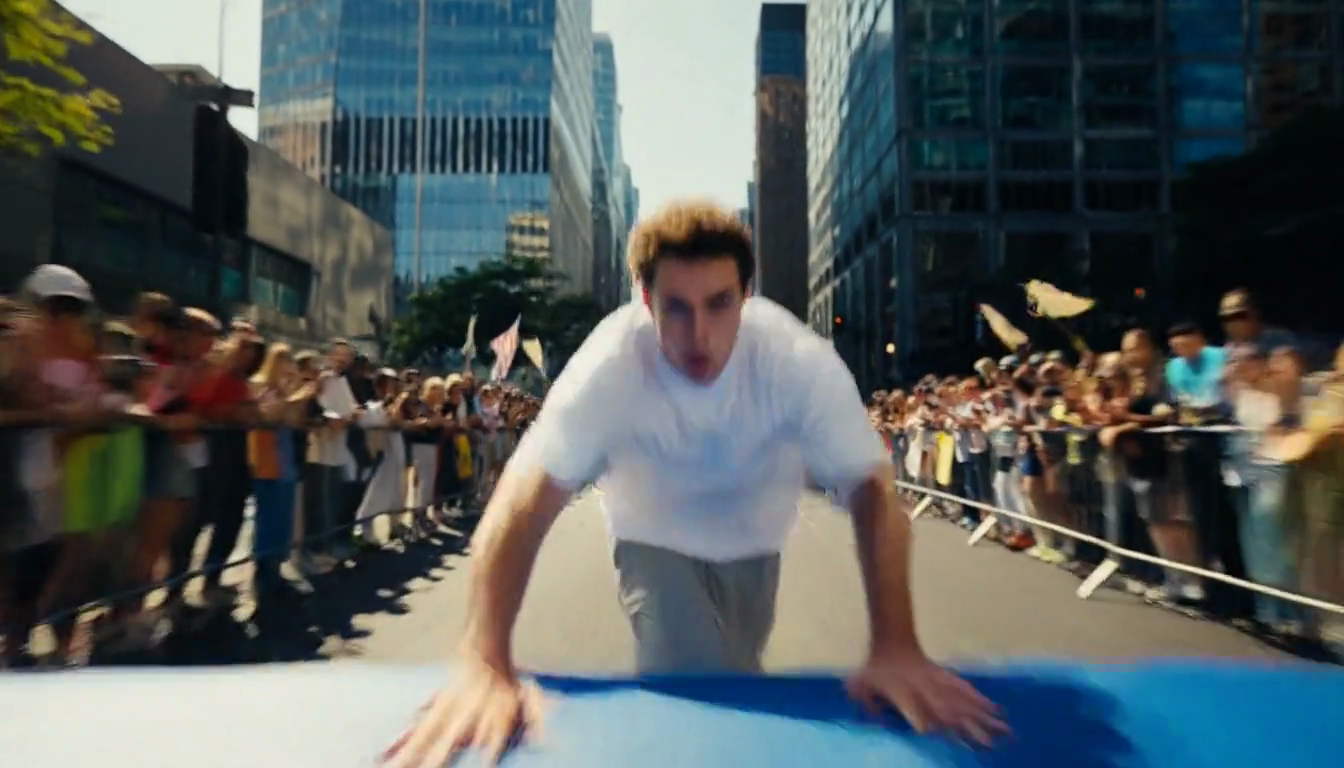} & \tframe{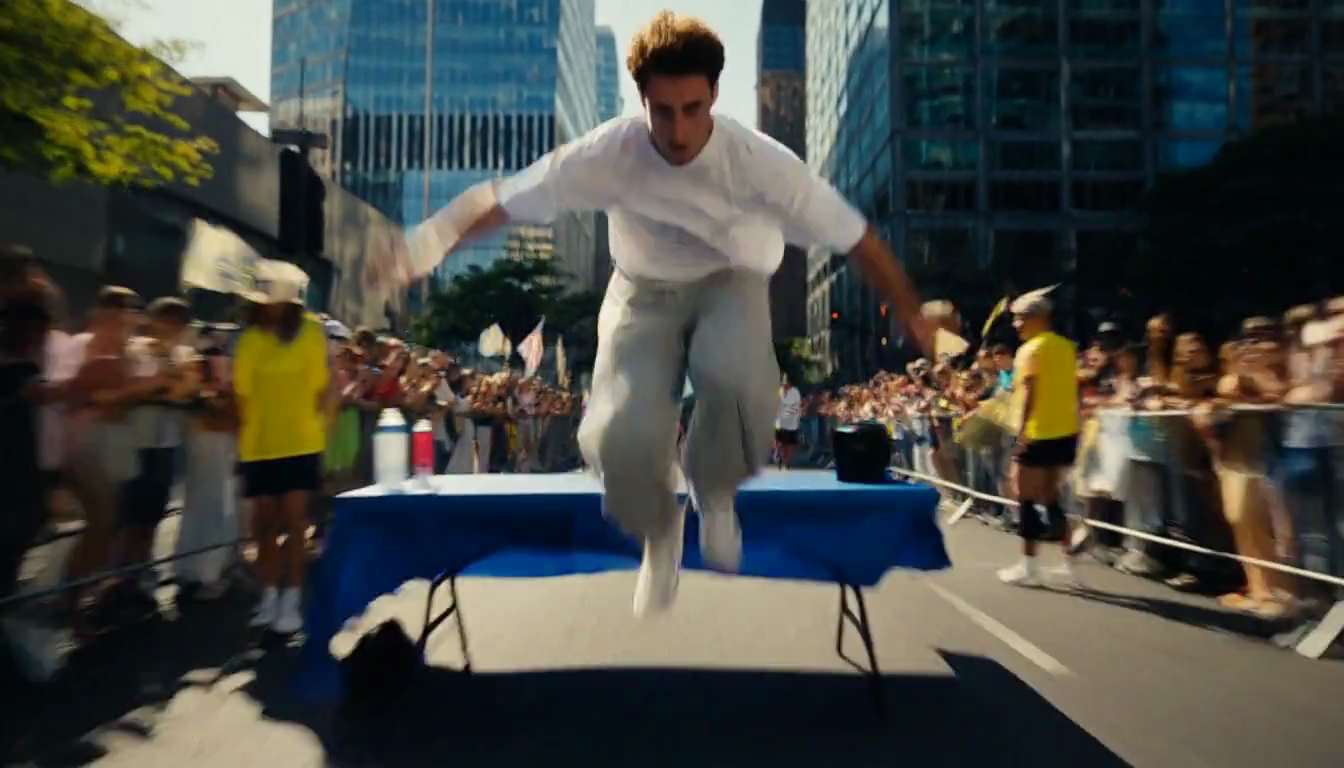} & \tframe{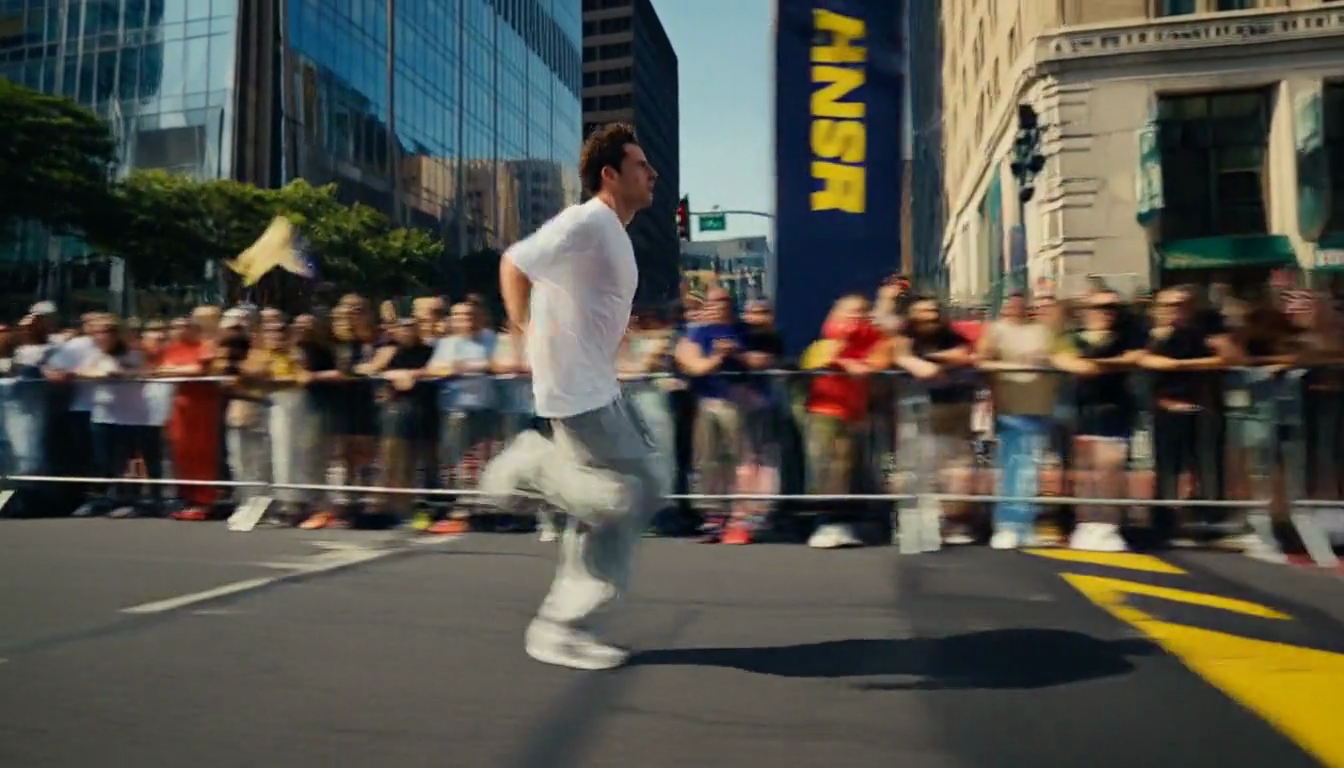} &
  \tframe{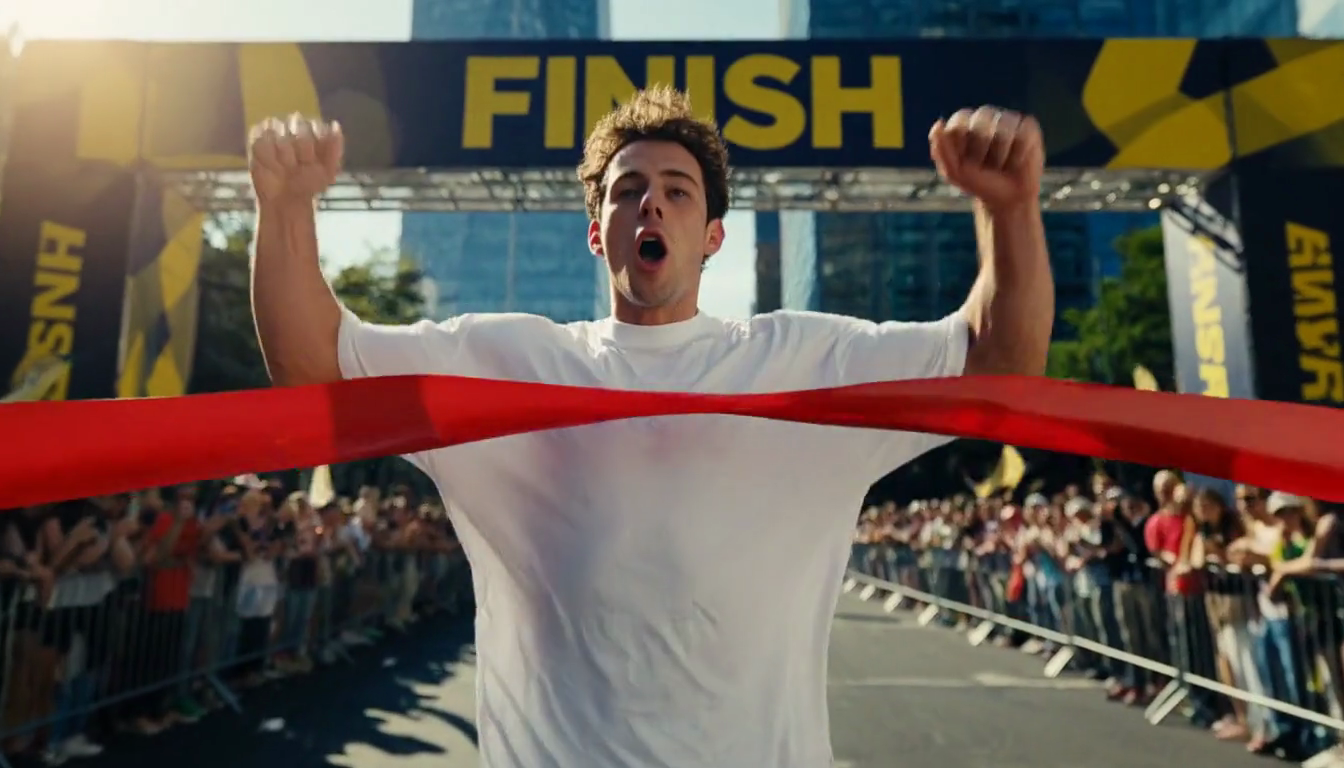}}
\newcommand{\TeaserRowOurs}{%
  \tframe{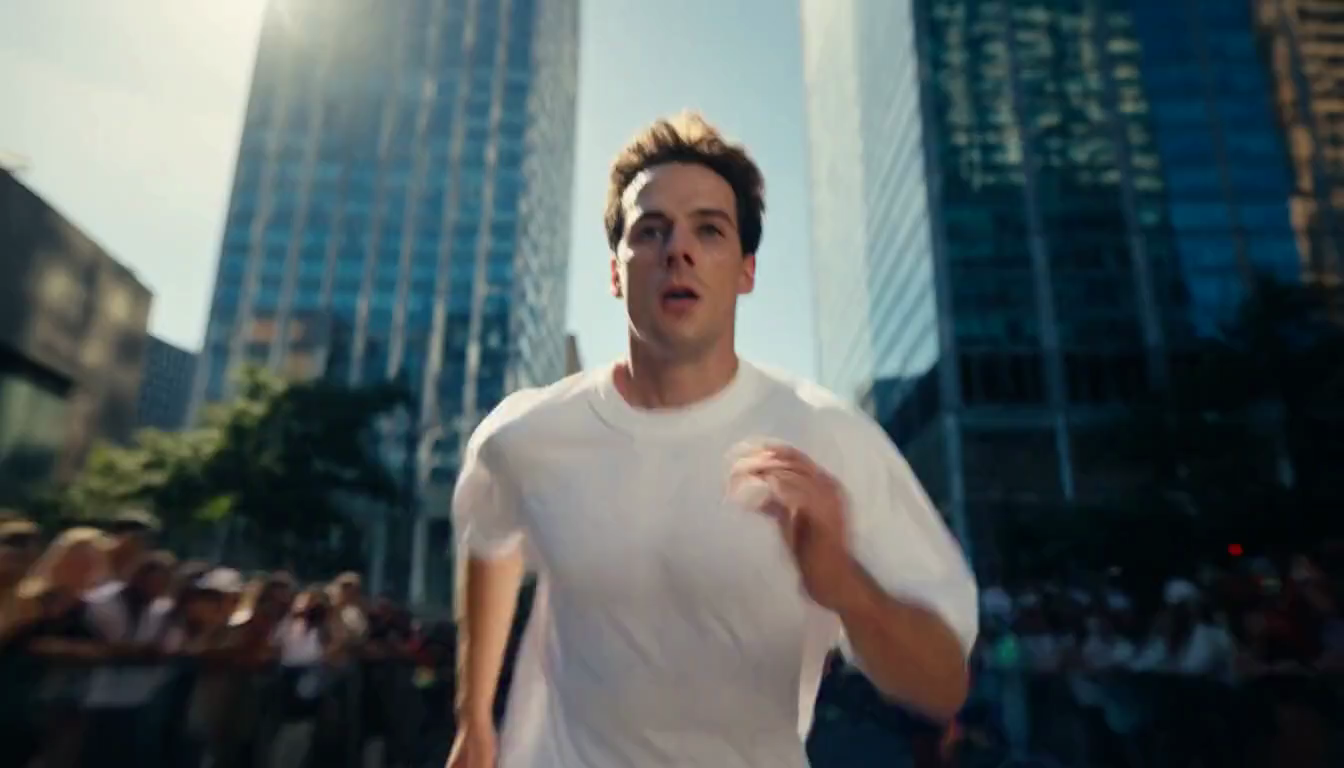} & \tframe{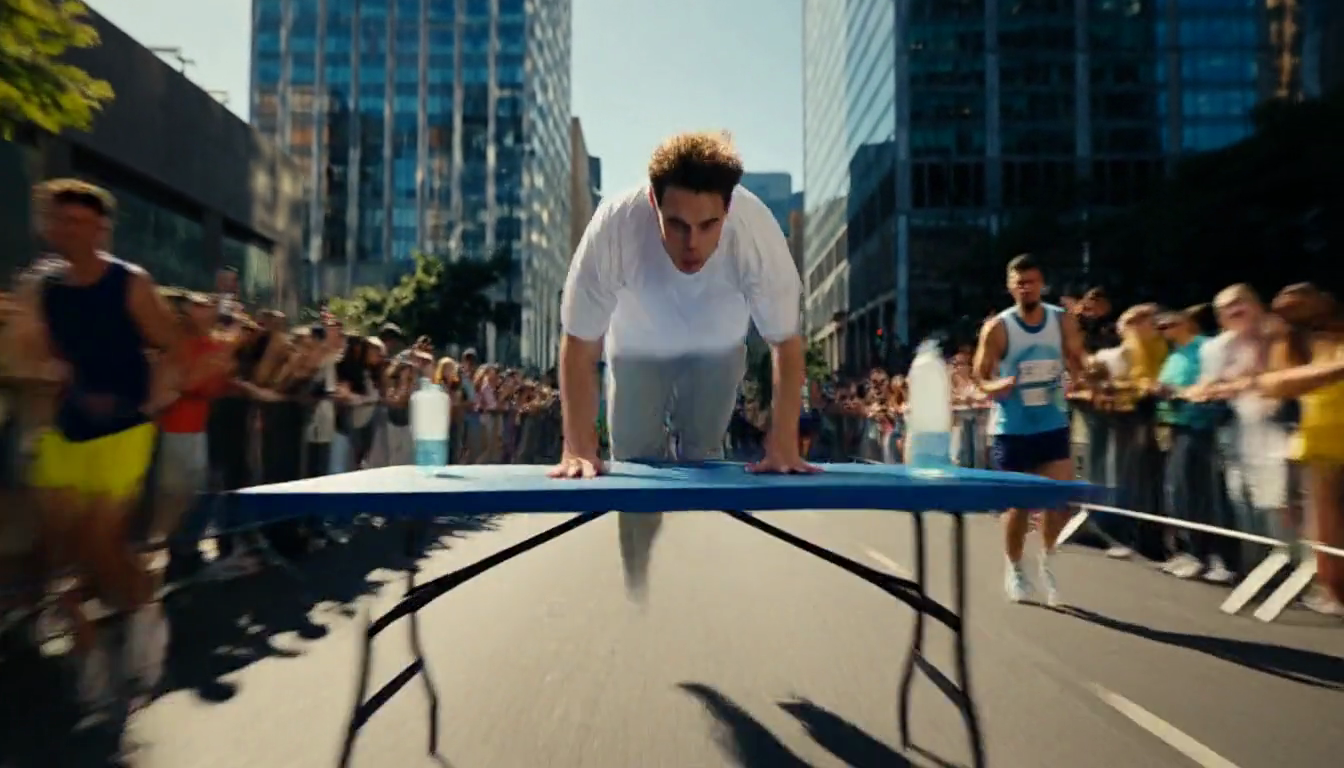} & \tframe{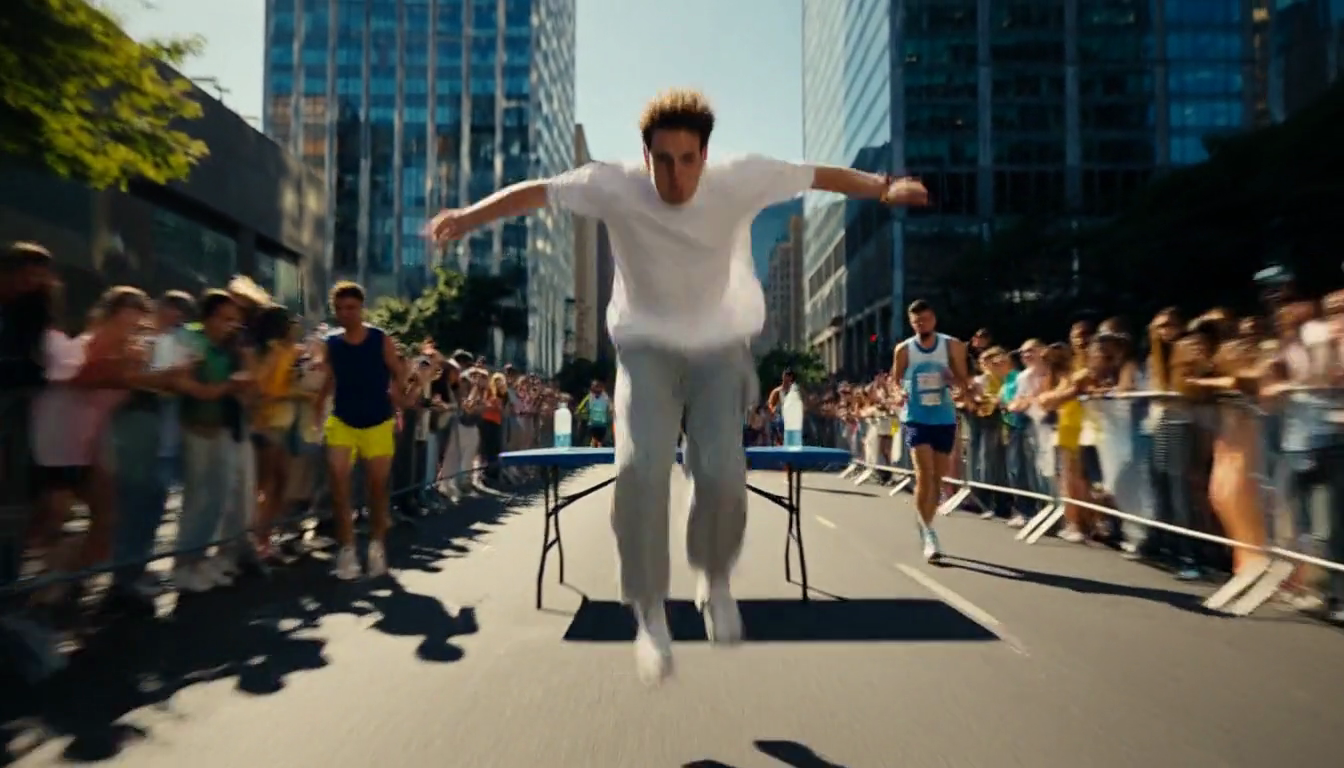} & \tframe{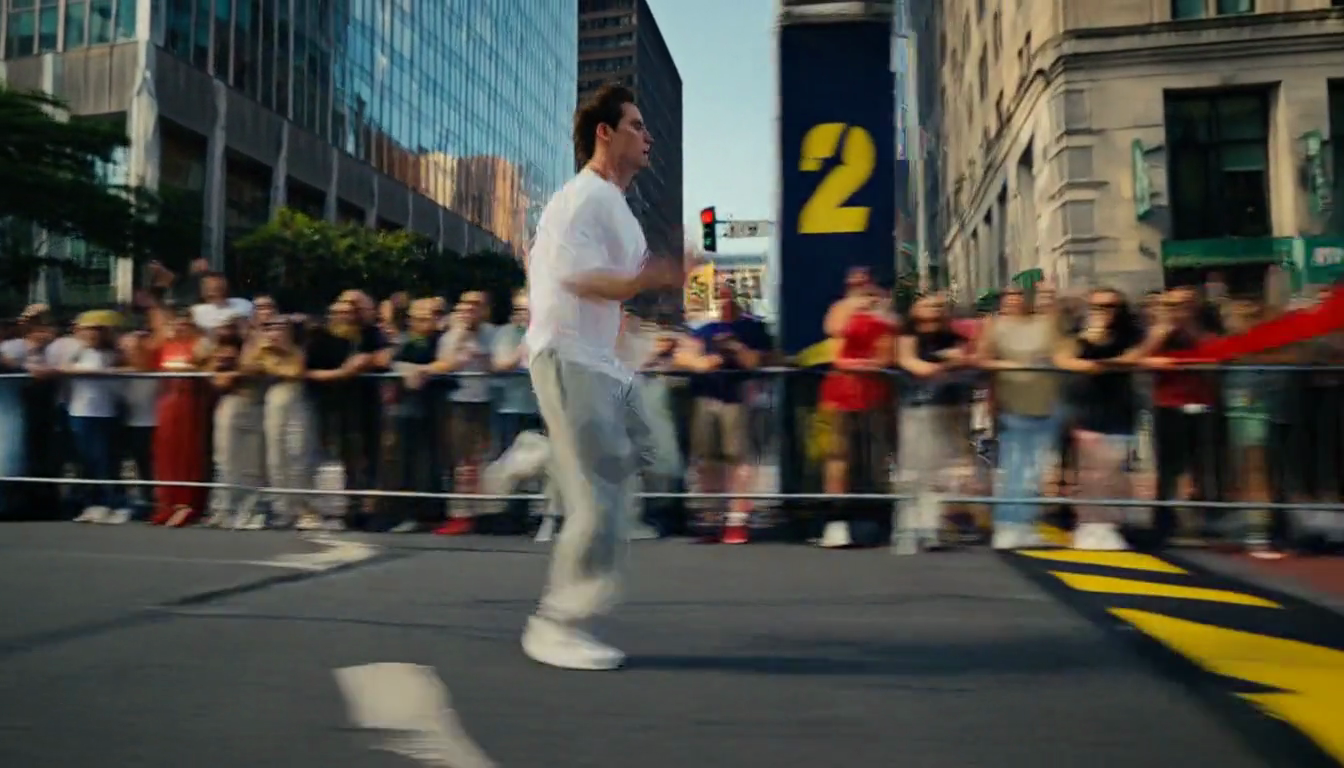} &
  \tframe{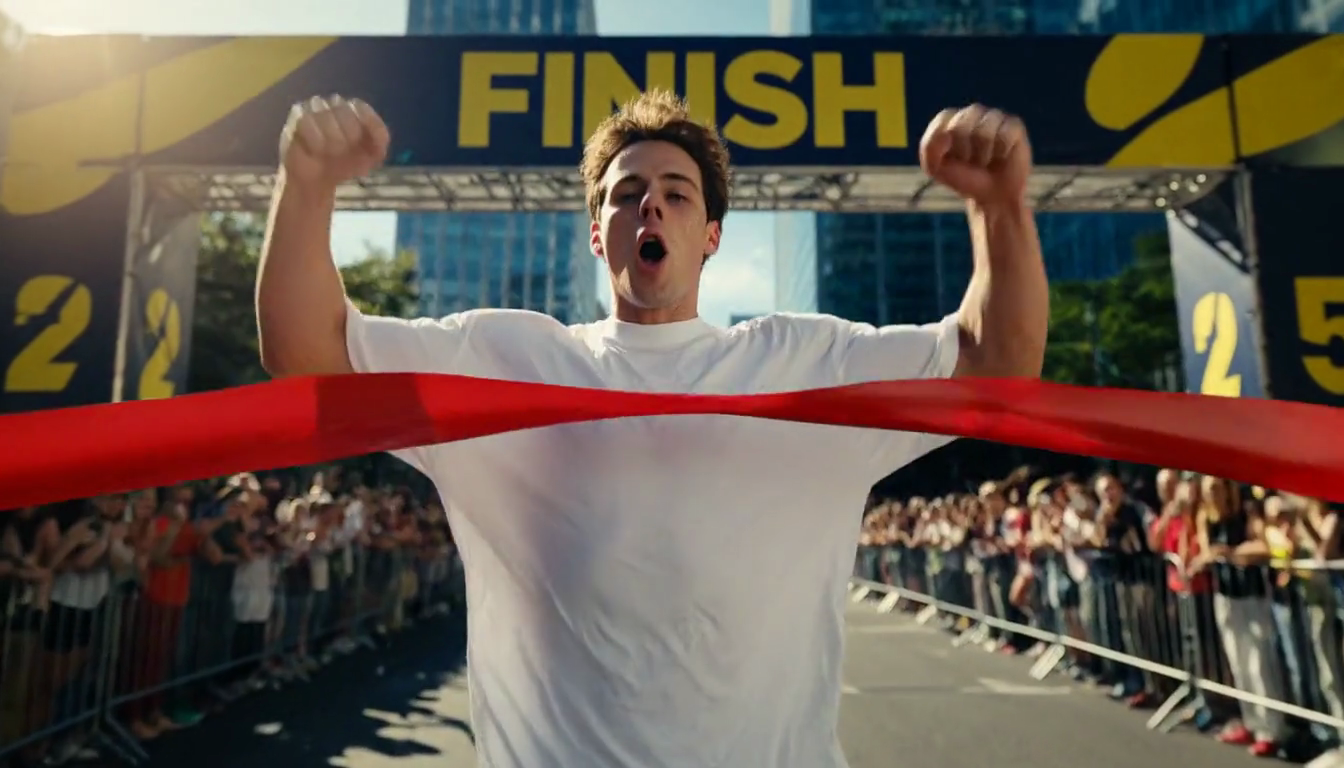}}
\newcommand{\tframe}[1]{%
  \fbox{\includegraphics[width=\TeaserFrameW]{Figures/teaser/#1.png}}}
\newcommand{\metricb}[2]{\textbf{#1}: \textbf{#2}}
\newcommand{\lab}[2]{\makebox[\linewidth][s]{#1\hfill #2}}
\newlength{\TeaserPSNRW}\newlength{\TeaserSpeedW}
\newcommand{\TeaserMetrics}[2]{\makebox[\TeaserPSNRW][l]{#1}\makebox[\TeaserSpeedW][l]{#2}}

\newlength{\qualwidth}
\newsavebox{\qualpanelbox}
\newcommand{\qualrow}[2]{%
  \lab{#1}{#2}\\[0.5pt]}
\newcommand{\qualimg}[1]{%
  \includegraphics[width=\qualwidth]{Figures/qualitative/#1.jpg}}

\newcommand{\FourBitGainRange}{0.7--4.0~dB\xspace}
\newcommand{\ExpCastFourBitCost}{0.5~dB\xspace}
\newcommand{\SageThreeAttnPro}{2.30\texttimes\xspace}
\newcommand{\SageThreeDenoisePro}{1.48\texttimes\xspace}
\newcommand{\OursPrepHThree}{5.8~ms\xspace}
\newcommand{\SageThreePrepHThree}{44.2~ms\xspace}
\newcommand{\OursPrepWanPro}{4.3~ms\xspace}
\newcommand{\SageThreePrepWanPro}{28.5~ms\xspace}
\newcommand{\SageThreeAttnFiftyNinety}{3.61\texttimes\xspace}
\newcommand{\SageThreeDenoiseFiftyNinety}{1.86\texttimes\xspace}%
\newcommand{\AttnSpeedupBThreeHundredFour}{2.07\texttimes\xspace}
\newcommand{\DenoiseSpeedupBThreeHundredFour}{1.48\texttimes\xspace}
\newcommand{\DirectPFastAttnBThreeHundred}{1.81\texttimes\xspace}
\newcommand{\AttnQATAttnBThreeHundred}{1.36\texttimes\xspace}
\newcommand{\KernelOnlyGainLongCat}{2.8~dB\xspace}
\newcommand{\ExpCastCostAbl}{0.5--0.7~dB\xspace}
\newcommand{\ExpCastAblGain}{1.30\texttimes\xspace}
\newcommand{\PlacementFirstOverUniform}{1.0~dB\xspace}
\newcommand{\PlacementAllOverFirst}{0.5~dB\xspace}
\newcommand{\PlacementFirstOverUniformEightBit}{2.6\xspace}
\newcommand{\PlacementAllOverFirstEightBit}{0.5~dB\xspace}
\newcommand{\AttnQATFourBitBTwoHundred}{0.85\texttimes\xspace}
\newcommand{\DenoiseSpeedupBTwoHundredEight}{1.26\texttimes\xspace}

\definecolor{nxlinkblue}{rgb}{0.21,0.49,0.74}
\usepackage[breaklinks,colorlinks,citecolor=nxlinkblue,
            linkcolor=red,urlcolor=nxlinkblue]{hyperref}
\hypersetup{
  pdftitle={VC-Attention: Value Smoothing and Softmax Casting for Low-bit Attention},
  pdfauthor={Xingyang Li, Dongyun Zou, Shining Zhang, Jiacheng Chen, Haocheng Xi,
             Lvmin Zhang, Jun-Yan Zhu, Song Han, Zhekai Zhang, Yujun Lin, Muyang Li},
  pdfsubject={Low-bit attention for video diffusion transformers},
  pdfkeywords={low-bit attention, quantization, video diffusion, FP8, NVFP4}}
\usepackage[capitalize]{cleveref}
\crefname{section}{Sec.}{Secs.}
\Crefname{section}{Section}{Sections}
\crefname{table}{Tab.}{Tabs.}
\Crefname{table}{Table}{Tables}
\crefname{appendix}{App.}{Apps.}
\Crefname{appendix}{Appendix}{Appendices}

\nunchuxheader{VC-Attention: Value Smoothing and Softmax Casting for Low-bit Attention}

\title{\method: \underline{V}alue Smoothing and Softmax \underline{C}asting for Low-bit Attention}

\author{%
\textbf{Xingyang Li}$^{*2}$\quad
\textbf{Dongyun Zou}$^{*1}$\quad
\textbf{Shining Zhang}$^{*3}$\quad
\textbf{Jiacheng Chen}$^{1}$\quad
\textbf{Haocheng Xi}$^{4}$\quad
\textbf{Lvmin Zhang}$^{5}$\\[2pt]
\textbf{Jun-Yan Zhu}$^{1,3}$\quad
\textbf{Song Han}$^{2,6}$\quad
\textbf{Zhekai Zhang}$^{1}$\quad
\textbf{Yujun Lin}$^{1}$\quad
\textbf{Muyang Li}$^{1}$\\[4pt]
$^1$Nunchux AI\quad $^2$MIT\quad $^3$CMU\quad $^4$UC Berkeley\quad
$^5$Stanford\quad $^6$NVIDIA\quad
{\small $^*$Equal contribution}%
}

\begin{document}
\maketitle
\begin{strip}
  \vspace{\TeaserRaise}
  \centering
  \setlength{\tabcolsep}{0pt}
\setlength{\fboxsep}{0pt}
\setlength{\fboxrule}{0.3pt}
\renewcommand{\arraystretch}{1}
\setlength{\TeaserFrameW}{\dimexpr\linewidth/\TeaserCols-2\fboxrule\relax}
\scriptsize
\settowidth{\TeaserPSNRW}{\metricb{PSNR}{\TeaserPSNRVC}\quad}%
\settowidth{\TeaserSpeedW}{\metricb{Attention speedup}{1.00\texttimes}}%
\begin{tabular}{@{}*{\TeaserCols}{c}@{}}
  \multicolumn{\TeaserCols}{@{}c@{}}{\parbox{\linewidth}{\centering
    \TeaserPrompt}}\\[\TeaserPromptGap]
  \multicolumn{\TeaserCols}{@{}c@{}}{\lab{FlashAttention-4 (BF16)}{\TeaserMetrics{}{%
    \metric{Attention speedup}{1.00\texttimes}}}}\\[\TeaserLabelGap]
  \TeaserRowFA\\[\TeaserRowGap]
  \multicolumn{\TeaserCols}{@{}c@{}}{\lab{SageAttention2 (8-bit)}{\TeaserMetrics{\metric{PSNR}{\TeaserPSNRSage}}{%
    \metric{Attention speedup}{\TeaserSpeedupSage}}}}\\[\TeaserLabelGap]
  \TeaserRowSage\\[\TeaserRowGap]
  \multicolumn{\TeaserCols}{@{}c@{}}{\lab{\best{\method}~\textbf{(8-bit)}}{\TeaserMetrics{\metricb{PSNR}{\TeaserPSNRVC}}{%
    \metricb{Attention speedup}{\TeaserSpeedupVC}}}}\\[\TeaserLabelGap]
  \TeaserRowOurs
\end{tabular}

  \captionsetup{hypcap=false}
  \captionof{figure}{\textbf{\method computes attention \TeaserSpeedupVC faster than BF16
  FlashAttention-4 with high fidelity.} PSNR is over \NumPrompts MiniMax-H3 prompts at 1344\texttimes768 on one B200.
  SageAttention2 (middle), recompiled for B200, runs slower than BF16 and moves the runner, the table, and the gate text. \method (bottom) keeps them.
  The denoising loop runs \DenoiseSpeedupBTwoHundredEight faster.
  At 4 bits, attention runs \AttnSpeedupRTXPro, \AttnSpeedupRTXFiftyNinety, and \AttnSpeedupBThreeHundredFourBit faster on the RTX PRO 6000, RTX 5090, and B300, and the denoising loop \DenoiseSpeedupRTXPro, \DenoiseSpeedupRTXFiftyNinety, and \DenoiseSpeedupBThreeHundredFour faster (\cref{fig:efficiency}).}
  \captionsetup{hypcap=true}
  \label{fig:teaser}
  \vspace{\TeaserBelow}
\end{strip}

\begin{figure*}[t]
  \centering
  \resizebox{\linewidth}{!}{\colorlet{ovgrey}{nxgrey!18}
\colorlet{ovslate}{nxgrey}
\colorlet{ovcard}{nxblue}
\providecommand{\tn}[1]{\textbf{#1}}
\newcommand{\ovswatch}[1]{\tikz[baseline=-0.35ex]{%
  \filldraw[draw=nxdark, line width=0.3pt, fill=#1]
    (0,0) rectangle (1.5mm,1.5mm);}}
\begin{tikzpicture}[
  font=\fontsize{6}{7}\selectfont,
  grid/.style={draw=white, line width=0.7pt},
  gborder/.style={draw=ovslate, line width=0.8pt},
  cborder/.style={draw=ovcard, line width=0.8pt},
  stripes/.style={inner sep=0pt, minimum width=0.52cm, minimum height=0.80cm},
  tile/.style={inner sep=0pt, minimum width=0.60cm, minimum height=0.60cm},
  nlab/.style={inner sep=0pt, font=\fontsize{8}{9}\selectfont},
  glab/.style={anchor=north, align=center, inner sep=0pt, yshift=-2.4pt,
               font=\fontsize{6}{7}\selectfont},
  op/.style={inner sep=0pt, font=\large},
  flow/.style={-{Latex[length=1.4mm,width=1.1mm]}, line width=0.5pt, draw=black!80,
               shorten >=0.5pt, shorten <=0.5pt},
  opbox/.style={draw=none, fill=nxsoft, text=black, rounded corners=1.5pt,
                align=center, inner xsep=2.0pt, inner ysep=0pt,
                minimum height=0.80cm},
  optitle/.style={inner sep=0pt, font=\fontsize{7}{8}\bfseries\selectfont,
                  text=ovcard},
  ours/.style={rounded corners=2pt, draw=ovcard, dashed, line width=0.8pt,
               inner sep=0pt},
  ourstitle/.style={anchor=south, font=\fontsize{6}{7}\bfseries\selectfont,
                    text=ovcard, inner sep=1.5pt},
  verdict/.style={anchor=north, inner sep=0pt, align=left,
                  font=\fontsize{6}{7}\selectfont},
  panel/.style={anchor=north, inner sep=0pt, font=\scriptsize},
  node distance=0.16cm,
]
\def\rh{0.13333}\def\ny{0.64}\def\gy{-0.50}
\newcommand{\stripes}[3][ovcard]{%
  \foreach \f [count=\k from 0] in {#3}{%
    \fill[fill=\f] ($(#2.north west)+(0,-\k*\rh)$) rectangle
      ($(#2.north east)+(0,-\k*\rh-\rh)$);}
  \foreach \k in {1,...,5}{%
    \draw[grid] ($(#2.north west)+(0,-\k*\rh)$) -- ($(#2.north east)+(0,-\k*\rh)$);}
  \draw[draw=#1, line width=0.8pt] (#2.north west) rectangle ($(#2.north east)+(0,-6*\rh)$);}
\newcommand{\meanrows}[3]{%
  \filldraw[cborder, fill=#2] ($(#1.north west)+(0,-0.8*\rh)$) rectangle
    ($(#1.north east)+(0,-2.2*\rh)$);
  \filldraw[cborder, fill=#3] ($(#1.north west)+(0,-3.8*\rh)$) rectangle
    ($(#1.north east)+(0,-5.2*\rh)$);}
\newcommand{\gridtile}[1]{%
  \fill[fill=ovgrey] (#1.north west) rectangle ($(#1.north east)+(0,-0.60cm)$);
  \foreach \k in {1,2,3}{%
    \draw[grid] ($(#1.north west)+(0,-\k*0.15cm)$) -- ($(#1.north east)+(0,-\k*0.15cm)$);}
  \foreach \k in {1,...,5}{%
    \draw[grid] ($(#1.north west)+(\k*0.10cm,0)$) -- ($(#1.north west)+(\k*0.10cm,-0.60cm)$);}
  \draw[gborder] (#1.north west) rectangle ($(#1.north east)+(0,-0.60cm)$);}
\newcommand{\cmark}{\tikz[baseline=-0.55ex]{\fill[verdictgood] circle(0.85mm);
  \draw[white,line width=0.35pt] (-0.42mm,0.02mm) -- (-0.12mm,-0.32mm) -- (0.45mm,0.36mm);}}
\newcommand{\xmark}{\tikz[baseline=-0.55ex]{\fill[verdictbad] circle(0.85mm);
  \draw[white,line width=0.35pt] (-0.32mm,-0.32mm) -- (0.32mm,0.32mm)
    (-0.32mm,0.32mm) -- (0.32mm,-0.32mm);}}

\node[tile] (as) at (0,0) {};
\gridtile{as}
\node[nlab] at (as |- 0,\ny) {\tn{S} = \tn{QK}$^{\!\top}$};
\node[glab] at (as |- 0,\gy) {INT8 / FP4\\matmul};
\node[opbox, fill=nxgrey!22, right=0.40cm of as] (aexp) {exp(S), then cast};
\node[glab] at (aexp |- 0,\gy) {\textbf{slow} SFU\\\textbf{16} exps per clock};
\node[optitle, text=nxgrey] at (aexp |- 0,\ny) {Exp + cast};
\draw[flow] (as.east) -- (aexp.west);
\node[tile, right=0.40cm of aexp] (ap) {};
\gridtile{ap}
\node[nlab] at (ap |- 0,\ny) {\tn{P}};
\node[glab] at (ap |- 0,\gy) {FP8 / FP4\\probability};
\draw[flow] (aexp.east) -- (ap.west);
\node[op, right=of ap] (ax) {$\times$};
\node[stripes, right=of ax] (av) {};
\stripes[ovslate]{av}{nxgrey!55,nxgrey!16,nxgrey!20,nxgrey!50,nxgrey!18,nxgrey!52}
\node[nlab] at (av |- 0,\ny) {\tn{V}};
\node[glab] at (av |- 0,\gy) {FP8 / FP4\\token order};
\coordinate (aleft) at (as.west);
\coordinate (aright) at (av.east);
\node[verdict] at ($(aleft)!0.5!(aright)+(0,-1.30)$)
  {\xmark\ the slow SFU is the bottleneck\\[1pt]
   \xmark\ the whole \ovswatch{nxgrey!55}\ value is quantized: wide range, large error};
\node[panel] at ($(aleft)!0.5!(aright)+(0,-2.12)$) {(a) Prior low-bit attention (e.g.\ SageAttention)};

\draw[line width=1pt, dotted] ($(aright)+(0.80,1.18)$) -- ++(0,-3.44);

\node[tile, right=1.75cm of av] (bs) {};
\gridtile{bs}
\node[nlab] at (bs |- 0,\ny) {\tn{S} = \tn{QK}$'^{\top}$};
\node[glab] at (bs |- 0,\gy) {INT8 / FP4\\matmul};
\node[opbox, right=0.40cm of bs] (ecbox) {Round($\lambda$\,*\,S + $\gamma$)};
\node[glab] at (ecbox |- 0,\gy) {\textbf{fast} CUDA cores\\\textbf{128} FMAs per clock};
\node[optitle] at (ecbox |- 0,\ny) {ExpCast};
\draw[flow] (bs.east) -- (ecbox.west);
\node[tile, right=0.40cm of ecbox] (bp) {};
\gridtile{bp}
\node[nlab] at (bp |- 0,\ny) {\tn{P}};
\node[glab] at (bp |- 0,\gy) {FP8 / FP4\\probability};
\draw[flow] (ecbox.east) -- (bp.west);
\node[op, right=of bp] (bx) {$\times$};
\node[stripes, right=0.22cm of bx] (bv) {};
\stripes{bv}{nxsoft,nxsoft!94,nxsoft!88,ovcard!26,ovcard!24,ovcard!22}
\node[nlab] (nbv) at (bv |- 0,\ny) {\tn{V}[$\pi$]};
\node[op, right=0.30cm of bv] (beq) {$=$};
\node[stripes, right=0.30cm of beq] (bm) {};
\meanrows{bm}{nxsoft!94}{ovcard!24}
\node[nlab] (nbm) at (bm |- 0,\ny) {$\boldsymbol\mu_j$};
\node[op, right=0.30cm of bm] (bpl) {$+$};
\node[stripes, right=0.30cm of bpl] (br) {};
\stripes{br}{ovcard!12,ovcard!10,ovcard!14,ovcard!9,ovcard!8,ovcard!7}
\node[nlab] (nbr) at (br |- 0,\ny) {\tn{R}};
\node[ours, fit={($(nbv.north west)+(-0.13,0.08)$) ($(br.south east)+(0.13,-0.10)$)}]
  (vsbox) {};
\node[ourstitle] at (vsbox.north) {V-Smooth};
\node[glab] at (bv |- 0,\gy) {BF16\\$k$-means sorted};
\node[glab] at (br |- 0,\gy) {FP8 / FP4\\residual};
\node[glab] at (bm |- 0,\gy) {BF16\\block mean};
\node[verdict] at ($(bs.west |- 0,0)!0.5!(vsbox.east |- 0,0)+(0,-1.30)$)
  {\cmark\ one multiply-add replaces the slow exp\\[1pt]
   \cmark\ only the small \ovswatch{ovcard!12}\ residual is quantized: narrow range, small error};
\node[panel] at ($(bs.west |- 0,0)!0.5!(vsbox.east |- 0,0)+(0,-2.12)$) {(b) \method\ (ours)};
\end{tikzpicture}}
  \caption{\textbf{\method\ changes two steps of low-bit attention, one for accuracy and one for speed.}
(a) Prior kernels exponentiate in FP32 before the cast and quantize the values in token order.
(b) ExpCast writes the probability code from the score ($\lambda=8/\ln 2$ and $\gamma=120+\beta$ for E4M3). At 4 bits, the same integer gives the NVFP4 code and its per-16 scale with a shift and one hardware conversion (\cref{sec:method-softmax}). V-Smooth sorts the tokens so that each block mean takes their shared component and only the pale residual is quantized.}
  \label{fig:method}
\end{figure*}

\begin{abstract}
Diffusion Transformers generate high-quality videos, but attention over their long token sequences dominates every denoising step.
Attention computes two matrix products: queries times keys give the scores, and probabilities times values give the output.
Running both in 8 or 4 bits promises several times the BF16 speed, but two obstacles remain.
The first is accuracy.
Prior work smooths the outlier channels of queries and keys, but the values have no such channels and remain a major source of error.
The second is speed.
Once both products run in low bits, the exponential between them becomes the bottleneck, because it runs mainly on the scarce special function units.
We propose \textbf{\method}, a training-free low-bit attention kernel that addresses both.
\textbf{V-Smooth} reduces the value error.
We find that value tokens form groups of similar tokens across the video, each with a large common component.
V-Smooth gathers each group into the same blocks by $k$-means and subtracts the shared component as the block mean, so only a small residual is quantized.
\textbf{ExpCast} removes the exponential.
It approximates the FP8 or NVFP4 code of each probability from its score with one multiply-add on the abundant CUDA cores.
We also normalize each row by the sum of its quantized probabilities, so each output row keeps its correct scale.
On Wan2.2, LongCat-Video, HunyuanVideo-1.5, and MiniMax-H3, the 8-bit kernel matches the strongest baseline on every model and exceeds it on two.
It runs attention 1.59\texttimes\ and 1.50\texttimes\ faster than BF16 FlashAttention-4 on B200 and H200, and the denoising loop 1.26\texttimes\ and 1.25\texttimes\ faster.
The 4-bit kernel leads the strongest 4-bit baseline by \FourBitGainRange of PSNR on all four models.
It runs attention 3.67\texttimes, 5.24\texttimes, and 2.07\texttimes\ faster on the RTX PRO 6000, RTX 5090, and B300, and the denoising loop 1.85\texttimes, 2.14\texttimes, and 1.48\texttimes\ faster.
\end{abstract}

\section{Introduction}
\label{sec:introduction}

Video diffusion models~\cite{ho2022video,blattmann2023align} now generate high-resolution clips with coherent motion.
Open-weight Diffusion Transformers (DiTs)~\cite{peebles2023scalable,wan2025wan,team2025longcat,wu2025hunyuanvideo,minimax2026h3} approach the quality of commercial systems~\cite{brooks2024sora,polyak2024moviegen,gao2025seedance}, and they power applications from character animation~\cite{guo2024animatediff} to autoregressive long-video generation~\cite{yin2025causvid,huang2025selfforcing} and world simulation~\cite{brooks2024sora}.

These models flatten the latent video into one long token sequence and apply full self-attention at every layer, so at high resolutions and long durations attention dominates inference time~\cite{xi2025sparse,li2025radial}.
Attention has two matrix products.
The first multiplies the queries by the keys to obtain scores.
A softmax turns the scores into probabilities, and the second product multiplies them by the values.
Both products grow quadratically with the token count~\cite{dao2022flashattention}.
Low-bit attention kernels~\cite{zhang2025sageattention,zhang2024sageattention2,zhang2025sageattention3,shah2024flashattention,chen2024intflashattention} run them on low-bit Tensor Cores, which are twice as fast as BF16 at INT8 and FP8 on Hopper and Blackwell GPUs and four to eight times as fast at FP4 on Blackwell GPUs~\cite{nvidia2022hopper,nvidia2024blackwell,nvidia2025rtxblackwell}.

However, low-bit attention for video generation still faces two problems.
The first is accuracy.
A quantizer stores each block of a tensor as low-bit numbers that share one scale.
The scale must fit the largest entry, so a few large entries make the rounding coarse for the whole block~\cite{xiao2023smoothquant,lin2024awq}.
In queries and keys, such entries sit in a few channels, which SageAttention~\cite{zhang2025sageattention,zhang2024sageattention2,zhang2025sageattention3} smooths and FlashAttention-3~\cite{shah2024flashattention} spreads with a Hadamard rotation.
Both target outlier channels, but values have none.
A rotation therefore does not reduce the value error (\cref{sec:prelim-value}), and Attn-QAT~\cite{zhang2026attnqat} reduces it only by retraining the model.
The value error thus remains the barrier to higher fidelity.
The second problem is speed.
Between the two products, the kernel computes one exponential per score to obtain the probabilities and then casts them to the low-bit format.
The exponential runs mainly on the special function units, which have only one eighth of the throughput of the CUDA cores~\cite{zadouri2026flashattention4}.
Low-bit Tensor Cores shorten the two products but not the exponential, which becomes the bottleneck.
On B200, naive 8-bit FlashAttention-4 with the exponential is only \AttnSpeedupNaiveBTwoHundred faster than BF16 (\cref{fig:efficiency}a).

We propose \method (\cref{fig:method}), a training-free low-bit attention kernel that solves both problems with \textbf{V}alue smoothing and a probability \textbf{C}ast.
For accuracy, we find that value tokens form groups of similar tokens far apart in space and time, and each group shares a large common component.
\textbf{V-Smooth} gathers each group into the same quantization blocks with a lightweight $k$-means and permutes the keys and values together, which leaves the output unchanged.
It then subtracts the mean of each value block, which reduces the squared norm of the values by \VKMeansEnergy (\cref{tab:reorder}), quantizes only the small residual, and adds the means back in the same accumulator.
This lowers the quantization error of the values and brings the videos closer to those of full-precision attention.
For speed, \textbf{ExpCast} exploits that a floating-point code and an exponential are both built on powers of two.
It writes the FP8 or NVFP4 code of each probability directly from its score with one multiply-add on the CUDA cores and never computes the exponential (\cref{fig:expcast}).
Both run in one fused kernel, at 8 bits on Hopper and Blackwell GPUs and at 4 bits on Blackwell GPUs.

We evaluate \method on Wan2.2~\cite{wan2025wan}, LongCat-Video~\cite{team2025longcat}, HunyuanVideo-1.5~\cite{wu2025hunyuanvideo}, and MiniMax-H3~\cite{minimax2026h3}.
At 8 bits, it runs attention \AttnSpeedupBTwoHundred and \AttnSpeedupHTwoHundred faster than BF16 FlashAttention-4~\cite{zadouri2026flashattention4} on B200 and H200 and matches the fidelity of the strongest baseline on every model.
At 4 bits, it is the most accurate training-free attention on all four models, ahead of SageAttention3 by \FourBitGainRange of PSNR.
It runs attention \AttnSpeedupRTXPro, \AttnSpeedupRTXFiftyNinety, and \AttnSpeedupBThreeHundredFour faster on the RTX PRO 6000, RTX 5090, and B300, and the denoising loop \DenoiseSpeedupRTXPro, \DenoiseSpeedupRTXFiftyNinety, and \DenoiseSpeedupBThreeHundredFour faster (\cref{fig:efficiency}).
\cref{fig:teaser,fig:qual-main} and \cref{sec:appendix-qualitative} show visual examples.

\section{Related Work}
\label{sec:related-work}

\begin{figure*}[t]
  \centering
  \Asset[\linewidth]{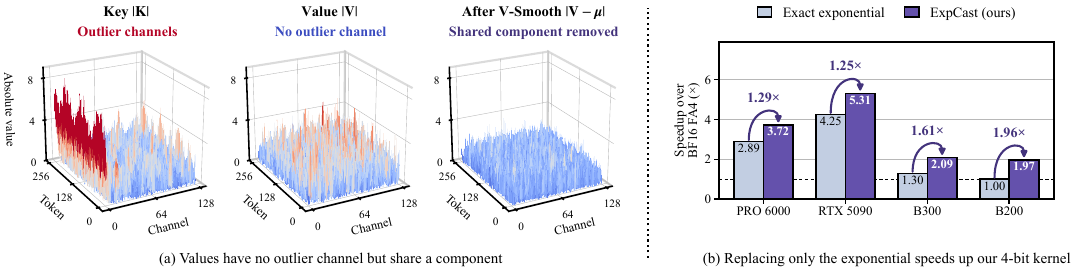}{1.75in}
  \caption{\textbf{Values have no outlier channel to smooth, and the exponential limits the kernel speed.}
  (a)~Absolute values of the keys and values in one HunyuanVideo-1.5 head (layer 30, 256 video tokens).
  The keys have a few outlier channels, which smoothing can remove.
  The values have none, but their tokens share a component, which V-Smooth subtracts.
  (b)~Speedup of our 4-bit kernel over BF16 FlashAttention-4 on one MiniMax-H3 call, with V-Smooth off.
  The two bars differ only in the exponential.}
  \label{fig:motivation}
\end{figure*}

\myparagraph{Efficient video generation.}
Video DiTs are expensive because they denoise a long spatiotemporal sequence many times.
One line of work shortens the trajectory, by distilling the model into a few-step generator~\cite{wang2023videolcm,li2024t2vturbo,yin2025causvid,ding2025dollar} or by reusing features across adjacent steps~\cite{ma2024deepcache,liu2025teacache}.
Another line shortens the sequence itself with deeply compressed autoencoders~\cite{chen2025deepcompression,chen2025dcvideogen,hacohen2024ltxvideo}, or spreads it across GPUs~\cite{li2024distrifusion,fang2024pipefusion}.
Closest to attention, sparse attention exploits the locality of video to skip most token pairs, with static~\cite{li2025radial,zhang2025fast}, predicted~\cite{xi2025sparse,yang2025sparse,zhang2025spargeattn,zhao2025harnessing}, or trained~\cite{zhang2025vsa} patterns, and linear attention replaces the softmax kernel~\cite{chen2025sana,wang2025lingen}.
Sparse VideoGen~2~\cite{yang2025sparse} and PAROAttention~\cite{zhao2025paroattention} also permute tokens, so that the attention map becomes block-sparse and, for PAROAttention, easier to quantize.
V-Smooth permutes the value tokens instead, so that the tokens in each quantization block are similar and their mean is large.
\method keeps all token pairs and steps, so it is orthogonal to them.

\myparagraph{Quantization.}
Post-training quantization of diffusion models has focused on the linear layers, from timestep-aware calibration of U-Nets~\cite{shang2023post,li2023q,wang2024accurate} to DiTs and video DiTs~\cite{wu2024ptq4dit,zhao2024vidit,tian2024qvd}.
Outliers are the central obstacle.
SmoothQuant and AWQ migrate them from activations into the weights~\cite{xiao2023smoothquant,lin2024awq}, QuaRot rotates them across channels~\cite{ashkboos2024quarot}, and SVDQuant absorbs them into a low-rank branch~\cite{lisvdquant}.
For video, DeltaQuant quantizes each token as a spatiotemporal cube mean plus a low-bit delta, and QuantVideoGen groups the KV cache of autoregressive video models and subtracts the group mean~\cite{li2026deltaquant,xi2026quantvideogen}.
However, all of these act on linear layers or stored caches, not on the operands of the attention Tensor Cores.

\myparagraph{Low-bit attention.}
Quantized attention feeds both operands of $\QK$ and $\PV$ to the Tensor Cores inside the tiled FlashAttention kernel~\cite{dao2022flashattention,dao2023flashattention2,shah2024flashattention}.
SageAttention smooths the keys and quantizes $\mathbf Q$ and $\mathbf K$ to INT8~\cite{zhang2025sageattention}, as does INT-FlashAttention~\cite{chen2024intflashattention}.
SageAttention2 quantizes them to INT4 or INT8 and the operands of $\PV$ to FP8~\cite{zhang2024sageattention2,zhang2025sageattention2pp}, and SageAttention3 quantizes both products to NVFP4~\cite{zhang2025sageattention3}.
FlashAttention-3 adds block quantization and a Hadamard rotation of the query and key to its FP8 path~\cite{shah2024flashattention}.
These methods reduce the probability error but leave the values in sequence order, and Attn-QAT recovers the remaining loss only by retraining~\cite{zhang2026attnqat}.
On the speed side, FlashAttention-4 and Attn-QAT find that softmax is the bottleneck of attention on B200~\cite{zadouri2026flashattention4,zhang2026attnqat}.
TurboAttention approximates the exponential with a lookup table and a polynomial before an INT8 cast~\cite{kang2024turboattention}.
Direct-P and EFQ-Softmax write FP4 probability codes directly, through affine maps fitted onto power-of-two block scales~\cite{hu2026hardware,han2026efqsoftmax}.
Schraudolph approximates an exponential in full precision by writing a linear function of the input into the exponent bits of a float~\cite{schraudolph1999fast}.
ExpCast instead writes the low-bit code that the Tensor Cores read, so one multiply-add replaces both the exponential and the cast, and it covers E4M3 and NVFP4 with one rule.

\section{\method}
\label{sec:method}

\subsection{Preliminaries and Motivation}
\label{sec:preliminaries}

\myparagraph{Online softmax.}\label{sec:prelim-low-bit-attention}
For $N$ spatiotemporal tokens and head dimension $d$, attention computes $\mathbf S=\mathbf Q\mathbf K^{\top}/\sqrt d$, $\mathbf P=\operatorname{softmax}(\mathbf S)$, and $\mathbf O=\mathbf P\mathbf V$.
A low-bit implementation replaces each operand by $\mathbf{X}_q=s_X\odot\widehat{\mathbf{X}}$, with one scale $s_X$ per quantization block.
FlashAttention streams the key and value tiles $j$ past a query tile $i$~\cite{dao2022flashattention}.
It raises the running row maximum from $m_i$ to $m_i'$ and accumulates
\begin{equation}
  \begin{aligned}
  \widetilde{\mathbf{P}}_{ij}&=\exp(\mathbf{S}_{ij}-m_i'), &\alpha_i&=\exp(m_i-m_i'),\\
  \mathbf{A}_i&\leftarrow\alpha_i\mathbf{A}_i+\widetilde{\mathbf{P}}_{ij}\mathbf{V}_j, &l_i&\leftarrow\alpha_i l_i+\widetilde{\mathbf{P}}_{ij}\mathbf{1},
  \end{aligned}
  \label{eq:tiled-low-bit-attention}
\end{equation}
then returns $\mathbf O_i=\mathbf A_i/l_i$.

\begin{figure}[t]
  \centering
  \includegraphics[width=\linewidth]{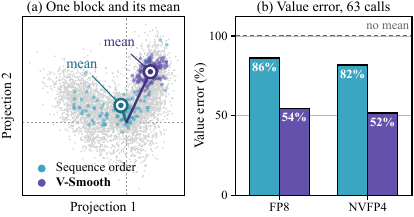}
  \caption{\textbf{Sorting puts similar value tokens into one block, so subtracting the block mean nearly halves the value error.}
  (a)~Value tokens of one Wan2.2 head on two directions. A 128-token block in sequence order spreads out, so its mean is near zero; after V-Smooth it is compact and its mean is large.
  (b)~Output error with only the values quantized, no mean subtracted being 100\%, on the 63 calls of \cref{tab:reorder}.}
  \label{fig:vsmooth}
\end{figure}

\myparagraph{Query-key smoothing leaves the value error untouched, and values share a component across the video.}\label{sec:prelim-value}
With $\mathbf{P}_q$ and $\mathbf{V}_q$ the reconstructed low-bit operands, the output error splits into a probability term and a value term,
\begin{equation}
  \mathbf{O}-\mathbf{P}_q\mathbf{V}_q = (\mathbf{P}-\mathbf{P}_q)\mathbf{V} + \mathbf{P}_q(\mathbf{V}-\mathbf{V}_q).
  \label{eq:attention-error-decomposition}
\end{equation}
Smoothing and Hadamard rotation of queries and keys shrink the first term~\cite{zhang2025sageattention,shah2024flashattention}.
They do not change the second term.
Smoothing and rotation help queries and keys because their large entries sit in a few channels.
Smoothing shifts these channels, and rotation spreads them over all channels, so the shared scale of each block becomes smaller~\cite{kvquant,zhang2025sageattention3}.
Values have no such outlier channels (\cref{fig:motivation}a), so a rotation does not reduce their error (\cref{tab:reorder}).
Without outliers, the quantization error grows with the squared norm of the values~\cite{chee2023quip,lisvdquant}.
In a video, value tokens form groups of similar tokens far apart, and each group shares a large common component.
If such tokens form a block, subtracting the block mean reduces their squared norm by \VKMeansEnergy, against \VSeqEnergy for blocks in sequence order and \VGaussEnergy for random Gaussian numbers grouped the same way (\cref{tab:reorder}).

\myparagraph{The exponential limits the speed of low-bit attention.}\label{sec:prelim-softmax}
$\QK$ and $\PV$ run on the Tensor Cores, while the exponential runs mainly on the special function units (SFUs).
The SFUs return 16 exponentials per streaming multiprocessor per clock~\cite{zadouri2026flashattention4}, against 128 multiply-adds on the CUDA cores~\cite{nvidia2025rtxblackwell}.
FlashAttention overlaps the two stages, so a tile takes as long as the slower one~\cite{shah2024flashattention}.
At head dimension 128, the exponentials of a tile take as long as its two FP8 matrix products on H200.
On B200, they take twice as long at 8 bits and \SFUOverMMAFourBit as long at 4 bits.
FlashAttention-4 emulates some of them with multiply-adds on B200~\cite{zadouri2026flashattention4}, but its 8-bit version runs only \AttnSpeedupNaiveBTwoHundred faster than BF16.
At 4 bits, the cast also becomes slower: before it, the kernel must find the maximum of every 16 probabilities and divide them by their shared scale.
Our 4-bit kernel with the exact exponential is then no faster than BF16 FlashAttention-4 on B200, and the released 4-bit kernel of Attn-QAT runs at \AttnQATFourBitBTwoHundred~\cite{zhang2026attnqat}.
Replacing only the exponential speeds up our kernel by \ExpOnlyGain (\cref{fig:motivation}b).

\subsection{V-Smooth: Value-Guided Token Smoothing}
\label{sec:method-value}

\begin{table}[t]
  \centering
  \caption{\textbf{Grouping by $k$-means is by far the most effective token order.} Removed: share (\%) of the squared norm removed by the 128-token block means. rMSE: output error with only $\mathbf V$ quantized (per-channel NVFP4), relative to no mean subtracted. Mean of 63 attention calls from the four models. Last row: Gaussian numbers with the per-channel mean and variance of the values. Overhead: grouping over attention time (\%) on Wan2.2 on the RTX PRO 6000.}
  \label{tab:reorder}
  {\setlength{\tabcolsep}{3pt}\footnotesize\renewcommand{\arraystretch}{1.2}
\begin{tabularx}{\linewidth}{lYYY}
\toprule
Token order & Removed\,$\uparrow$ & rMSE\,$\downarrow$ & Overhead \\
\midrule
No mean subtracted & 0 & 1.00 & 0 \\
\midrule
Sequence & 16.8 & 0.82 & 0 \\
Hadamard on $\mathbf V$ & 16.8 & 0.82 & 1.6 \\
Static cube (DeltaQuant) & 21.9 & 0.76 & 3.6 \\
\ours $k$-means & 47.6 & 0.52 & 22.6 \\
Balanced $k$-means & 54.4 & 0.45 & 115.9 \\
\midrule
$k$-means, Gaussian numbers & 18.3 & -- & -- \\
\bottomrule
\end{tabularx}}

\end{table}
V-Smooth reduces the squared norm of the values before quantization (\cref{fig:vsmooth}).
It sorts similar value tokens into the same blocks, subtracts the mean of each block, quantizes what remains, and adds the means back inside the kernel.

\myparagraph{Subtracting the block mean removes what the tokens of a block have in common.}
The kernel loads the values in blocks of $B_v=\BlockSizeV$ tokens.
For each block $j$, V-Smooth computes the mean $\boldsymbol\mu_j$ and the residual $\mathbf R_j$, and the kernel quantizes $\mathbf R_j$ in place of the block, with per-channel E4M3 at 8 bits and NVFP4 at 4 bits:
\begin{equation}
  \boldsymbol{\mu}_j
  =\tfrac{1}{B_v}{\mathbf{V}'_j}^{\top}\mathbf{1},\qquad
  \mathbf{R}_j=\mathbf{V}'_j-\mathbf{1}\boldsymbol{\mu}_j^{\top}.
  \label{eq:block-demeaning}
\end{equation}
A residual with a smaller squared norm has a smaller quantization error (\cref{sec:prelim-value}).
Its squared norm is
\begin{equation}
  \|\mathbf{R}_j\|_F^2=\|\mathbf{V}'_j\|_F^2-B_v\|\boldsymbol\mu_j\|_2^2.
  \label{eq:optimal-value-mean}
\end{equation}
The subtracted term is large only when the tokens of the block are similar, so that their mean is large.
In sequence order, neighboring tokens are only weakly similar, and the block mean reduces the squared norm by only \VSeqEnergy (\cref{tab:reorder}).
The one mean over all tokens that SageAttention2 subtracts removes even less, because no vector reduces the squared norm of a block more than the block's own mean.

\myparagraph{Value-guided permutation.}
To bring similar tokens into one block, an online $k$-means over the value tokens of each batch and head assigns a label $g_t$ to every token.
Sorting the labels yields a permutation $\pi=\operatorname{argsort}(g)$, applied to $\mathbf K$ and $\mathbf V$ while $\mathbf Q$ keeps its order:
\begin{equation}
  \begin{gathered}
  \mathbf{K}'=\mathbf{K}[\pi],\quad
  \mathbf{V}'=\mathbf{V}[\pi],\\
  \mathbf{P}'=\operatorname{softmax}\!\left(\mathbf{Q}{\mathbf{K}'}^\top/\sqrt d\right).
  \end{gathered}
  \label{eq:kv-reordering}
\end{equation}
Permuting the keys permutes the columns of $\mathbf P$ exactly as the rows of $\mathbf V$.
Thus $\mathbf P'\mathbf V'=\mathbf P\mathbf V$, and the output of the non-causal self-attention of video DiTs is unchanged.
After sorting, block means reduce the squared norm of the values by \VKMeansEnergy (\cref{tab:reorder}), and cut the value error by 37\% against sequence order at both precisions (\cref{fig:vsmooth}b).
Balancing every cluster to exactly $B_v$ tokens lowers the value error by only \BalancedGain more at \BalancedCost the grouping time, so V-Smooth uses plain $k$-means (\cref{tab:reorder}).

\myparagraph{Online mean restoration.}
The kernel adds the means back inside its main loop, as the probabilities are never stored.
Substituting $\mathbf V'_j=\mathbf R_j+\mathbf 1\boldsymbol\mu_j^{\top}$ into \cref{eq:tiled-low-bit-attention} splits the value product into a low-bit and a rank-one part:
\begin{equation}
  \begin{gathered}
  \mathbf{A}_i \leftarrow \alpha_i\mathbf{A}_i
    +\widetilde{\mathbf{P}}_{q,ij}\,\mathbf{R}_{q,j}
    +r_{ij}\boldsymbol{\mu}_j^\top,
  \\
  l_i\leftarrow\alpha_i l_i+r_{ij},
  \end{gathered}
  \label{eq:vc-online-update}
\end{equation}
where $\widetilde{\mathbf P}_{q,ij}$ is the low-bit probability tile on the permuted keys, $\mathbf R_{q,j}$ is the quantized residual, and $r_{ij}$ is the amount added to the row sum $l_i$.
The mean term enters the same accumulator as the low-bit product, so the rescaling by $\alpha_i$ covers it and no extra buffer is needed.
Each mean costs one 16-bit vector per block, or only 0.125 bit per value element.
We set $r_{ij}=\widetilde{\mathbf P}_{q,ij}\mathbf 1$, so the kernel divides each output row by the sum of the quantized probabilities it multiplies by.
The weights of each row then add up to exactly one, and the restored means come back without error.
SageAttention3 instead multiplies the values by the 4-bit probabilities but divides the output by the sum of the FP32 probabilities.
The two do not match, so almost every output row comes out \SageConstVShrink too small (\cref{sec:appendix-fp8}).
With this choice and without V-Smooth, our kernel falls from \DenOursPSNR to \DenFPPSNR~dB PSNR on LongCat-Video, close to the \DenSagePSNR~dB of SageAttention3.
\cref{sec:method-softmax} shows that the same choice also bounds the error of approximate probabilities.

\myparagraph{V-Smooth schedule.}
V-Smooth runs on the first \GroupingStepFraction of the denoising steps, the span that Sparse VideoGen and Radial Attention keep dense~\cite{xi2025sparse,li2025radial}, and \cref{fig:placement} measures the placement.
Over a 40-step Wan2.2 clip on one B200, the schedule adds \GroupingOverhead to the average attention call, with preprocessing, against \GroupingOverheadEveryStep on every step (\cref{sec:appendix-impl}).

\subsection{ExpCast: Exponential-Free Low-Bit Softmax}
\label{sec:method-softmax}
The low-bit softmax computes an exponential and casts the result to FP8 or NVFP4, and the exponential limits its speed (\cref{sec:prelim-softmax}).
ExpCast replaces the exponential with one multiply-add that writes the bits of the result directly (\cref{fig:expcast}).
The hardware computes exponentials in base 2, so we write each probability as $2^z$, where $z\leq0$ is the score minus the row maximum, times $\log_2 e$.
\begin{figure}[t]
  \centering
  \definecolor{convband}{HTML}{7FC2CC}%
  \newcommand{\expswatch}[2]{\tikz[baseline=-0.6ex]\draw[line width=#2, color=#1] (0,0) -- (3.2mm,0);}%
  \begin{lrbox}{\expcastconv}\begin{minipage}{0.43\linewidth}
    {\scriptsize\expswatch{convband}{2.6pt}\,\textbf{Exponentiate, then cast}}
\begin{lstlisting}[style=nxcode,morecomment={[l]\#}]
p  = exp2(z+8)  # SFU
p8 = to_e4m3(p)
\end{lstlisting}
  \end{minipage}\end{lrbox}%
  \begin{lrbox}{\expcastours}\begin{minipage}{0.43\linewidth}
    {\scriptsize\expswatch{nxblue}{1pt}\,{\color{nxdark}\textbf{ExpCast-FP8 (ours)}}}
\begin{lstlisting}[style=nxcode,morecomment={[l]\#},literate={beta}{{$\beta$}}1]
c  = round(8*(z+15)+beta)
p8 = as_e4m3(c)  # no-op
\end{lstlisting}
  \end{minipage}\end{lrbox}%
  \begin{tikzpicture}[box/.style={line width=0.6pt, rounded corners=2.5pt, inner sep=3.1pt,
                                  anchor=north west, minimum width=0.45\linewidth}]
    \node[box, draw=nxgrey!45, fill=white] (conv) at (0,0) {\usebox{\expcastconv}};
    \node[box, draw=nxblue, fill=nxpale] (ours) at ([yshift=-3.0mm]conv.south west) {\usebox{\expcastours}};
    \path let \p1=(conv.north), \p2=(ours.south east),
              \n1={(\y1-\y2)*1.235/0.95}, \n2={\y2-(\y1-\y2)*0.245/0.95} in
      node[anchor=south west, inner sep=0pt] at (\x2+1.2mm, \n2)
        {\includegraphics[height=\n1]{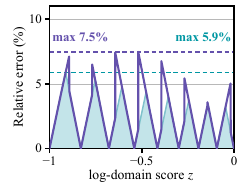}};
    \node[anchor=north, font=\scriptsize\bfseries, text=nxdark, inner sep=0pt]
      at ([yshift=-3.0mm]ours.south) {one FMA, no exponential, no cast};
  \end{tikzpicture}
  \caption{\textbf{One multiply-add replaces the exponential and the E4M3 cast.}
  ExpCast-FP8 writes the same byte on \SameByteShare of scores and is one code away on the rest.
  Right: the relative error over one doubling of the probability.}
  \label{fig:expcast}
\end{figure}

\begin{figure*}[t]%
  \centering\scriptsize
  \setlength{\qualwidth}{0.482\textwidth}%
  \sbox\qualpanelbox{\begin{minipage}[t]{\qualwidth}
    \centering
    \qualrow{FlashAttention-4 (BF16)}{\metric{PSNR}{reference}}
    \qualimg{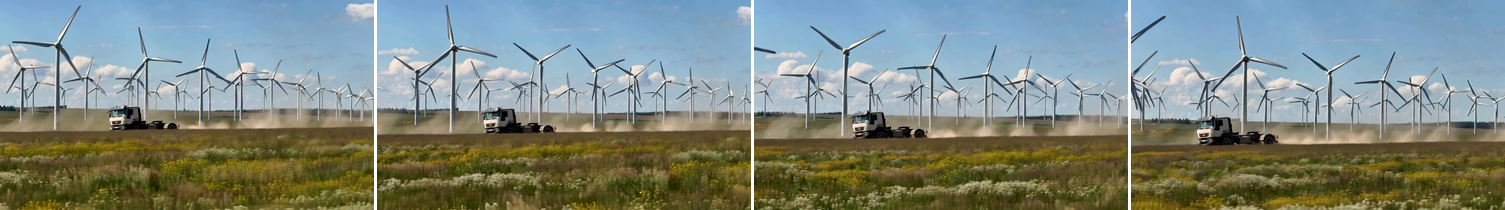}\\[2pt]
    \qualrow{SageAttention2 (8-bit)}{\metric{PSNR}{16.1\,dB}}
    \qualimg{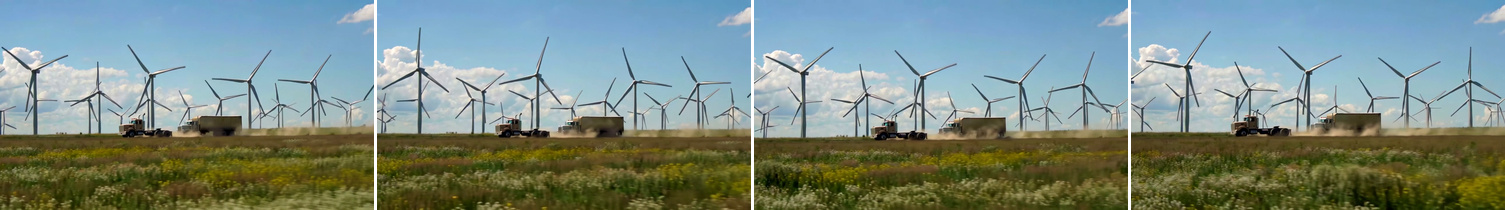}\\[2pt]
    \qualrow{\best{\method}~\textbf{(8-bit)}}{\best{\metricb{PSNR}{18.6\,dB}}}
    \qualimg{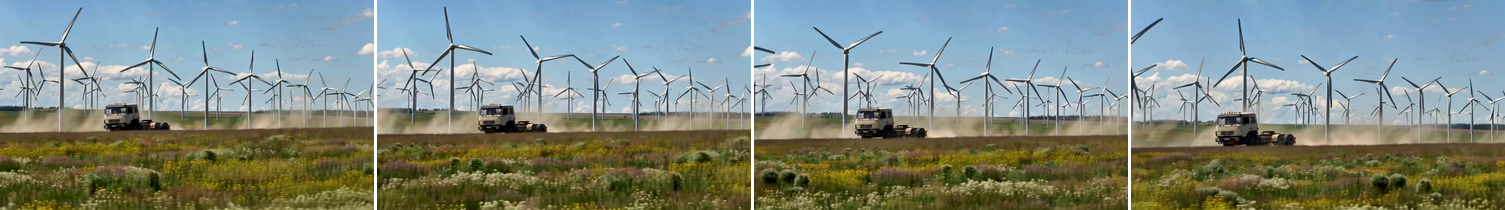}\\[2pt]
    {\scriptsize\itshape A truck drives past a row of wind turbines.}\\[5pt]
    {\footnotesize (a) HunyuanVideo-1.5, 8-bit, B200}
  \end{minipage}}%
  \usebox\qualpanelbox\hfill
  \tikz[baseline=0pt]{\draw[dash pattern=on 1.2pt off 1.6pt, line width=0.6pt] (0,\ht\qualpanelbox) -- (0,-\dp\qualpanelbox);}\hfill
  \begin{minipage}[t]{\qualwidth}
    \centering
    \qualrow{FlashAttention-4 (BF16)}{\metric{PSNR}{reference}}
    \qualimg{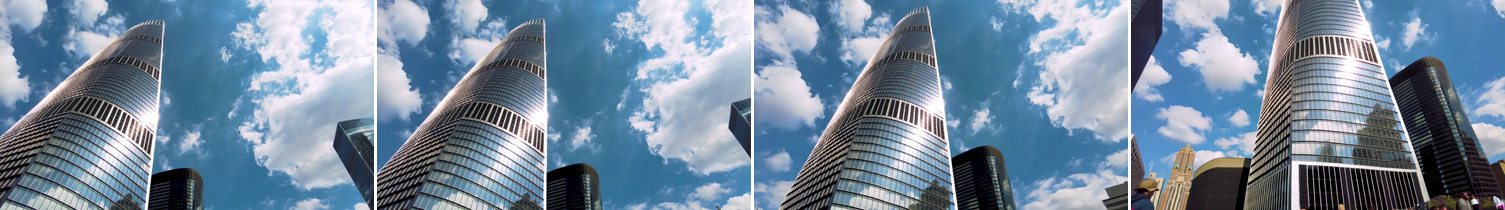}\\[2pt]
    \qualrow{SageAttention3 (4-bit)}{\metric{PSNR}{10.3\,dB}}
    \qualimg{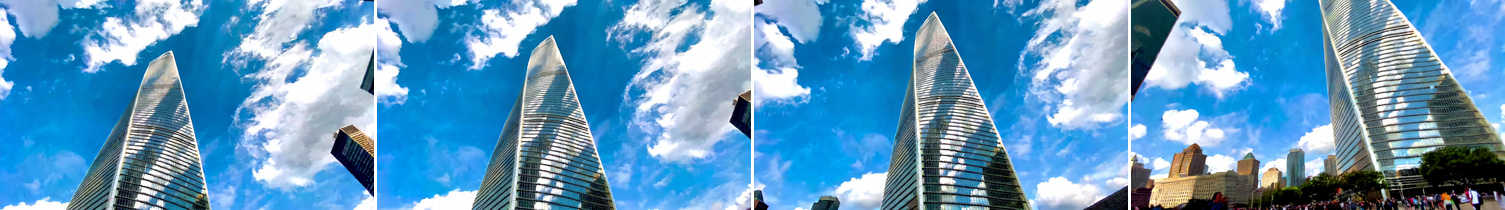}\\[2pt]
    \qualrow{\best{\method}~\textbf{(4-bit)}}{\best{\metricb{PSNR}{15.2\,dB}}}
    \qualimg{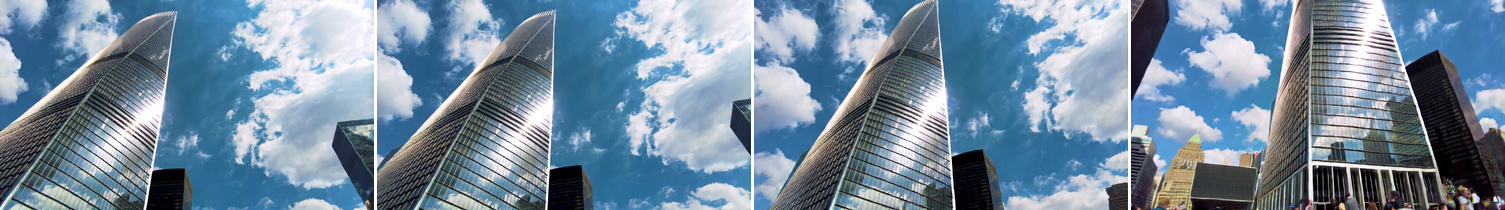}\\[2pt]
    {\scriptsize\itshape A low-angle shot of a skyscraper against a blue sky.}\\[5pt]
    {\footnotesize (b) HunyuanVideo-1.5, 4-bit, RTX PRO 6000}
  \end{minipage}
  \caption{\textbf{SageAttention changes what the video shows, and \method keeps it.}
  Both 720p clips are typical: their PSNR against BF16 FlashAttention-4 lies within 1.2~dB of the model means in \cref{tab:main}.
  SageAttention2 adds a second truck that the prompt does not ask for (a), and SageAttention3 replaces the tower of the BF16 video with a different one (b). \method keeps the framing and the subjects of the BF16 video in both clips.}
  \label{fig:qual-main}
\end{figure*}

\myparagraph{In FP8, the bits of a probability are almost a linear function of $z$.}
An FP8 (E4M3) number has a 4-bit exponent $E$ and a 3-bit mantissa $M$, and its value is $2^{E-7}(1+M/8)$.
We call the 8 bits, read as the integer $8E+M$, the \emph{code}.
E4M3 cannot store values below $2^{-9}$.
As in FlashAttention-4, the kernel therefore casts $2^{z+8}$ instead of $2^z$, so that small probabilities do not become zero.
The factor $2^8$ is the same for the whole row, so it cancels when the output is divided by the row sum.
Adding 1 to $E$ doubles the value, so for $2^{z+8}$, $E$ is the integer part of $z$ plus 15.
The mantissa stores the rest, $1+M/8\approx2^f$, where $f$ is the fractional part of $z$.
Between 0 and 1, $2^f$ is close to the straight line $1+f$~\cite{mitchell1962computer,schraudolph1999fast}, so $M\approx8f$.
The code is therefore about $8(z+15)$, and ExpCast-FP8 computes it with one multiply-add:
\begin{equation}
  \begin{gathered}
  c(z)=\operatorname{Round}\!\left(
    \operatorname{clip}_{[0,120]}\bigl(8(z+15)+\beta\bigr)\right),\\
  \beta=-0.35.
  \end{gathered}
  \label{eq:direct-fp8}
\end{equation}
For example, $z=-2.7$ gives $c=98$ ($E=12$, $M=2$), which decodes to $40$, close to the exact $2^{5.3}=39.4$.
The line $1+f$ is never below $2^f$, so the linear code is slightly too large, and $\beta$ shifts it down to center the error (\cref{sec:appendix-fp8}).
At the row maximum, where $z=0$, the code is 120, so the clip only sends very small probabilities to code 0, which decodes to zero.
ExpCast-FP8 writes the same byte as the exponential and the cast on \SameByteShare of scores and is one code away on the rest (\cref{fig:expcast}).

\myparagraph{NVFP4 cannot store small probabilities, so ExpCast-FP4 first writes them in FP16.}
NVFP4 stores each value in 4 bits, as one of 0, 0.5, 1, 1.5, 2, 3, 4, and 6, and shares one scale among 16 values.
From 1 to 6, every two steps double the value, so the code is again nearly linear in $z$.
Below 0.5, however, the next value is 0, so a linear rule cannot write small probabilities.
Direct-P fits a straight line onto them anyway~\cite{hu2026hardware}.
It thus rounds many small probabilities to zero and loses \DirectPZeroMass of the softmax mass (\cref{sec:appendix-directp}).
ExpCast-FP4 instead applies the linear rule to FP16, the 16-bit floating-point format, whose values keep doubling down to $2^{-14}$.
It writes the scale of each group of 16 with the same rule, from the largest score of the group.
One hardware instruction then rounds each FP16 value to the nearest NVFP4 value (\cref{sec:appendix-impl}).
The probability error of ExpCast-FP4 is thus close to that of the exact exponential and less than half that of Direct-P.

\myparagraph{ExpCast stays close to the exact exponential.}
We compare ExpCast against the exact exponential followed by the round-to-nearest cast that every low-bit kernel uses.
We measure the error of a row by its total variation (TV), the share of the attention weight that moves between tokens.
TV bounds the error of an output row $\mathbf o$:
\begin{equation}
  \|\widehat{\mathbf o}-\mathbf o\|_2\leq\operatorname{TV}\cdot D,
  \label{eq:direct-fp8-row-bound}
\end{equation}
where $\widehat{\mathbf o}$ is the ExpCast output and $D$ is the largest distance between two value rows (proof in \cref{sec:appendix-fp8}).
Moving 1\% of the weight between two value rows changes the output by at most 1\% of their distance, and a component shared by all value rows causes no error.
The bound requires the written probabilities of each row to sum to one, which the division by their sum in \cref{sec:method-value} ensures.
On real rows, ExpCast raises TV from \CastTV to \MeasuredTV at E4M3 (\cref{sec:appendix-fp8}).
At E4M3, every normal probability is written between 0.929 and 1.0745 times its exact value, which limits TV to \textbf{3.64\%}.
At NVFP4, the cast rounds the smallest probabilities to zero, so no such limit holds, but the mean TV still stays at \FourBitTV, against \FourBitCastTV for the exact exponential and \DirectPTVAccurate--\DirectPTVFast for Direct-P (\cref{sec:appendix-directp}).

\section{Experiments}
\label{sec:experiments}
\begin{table*}[t]
  \centering
  \caption{\textbf{Fidelity and speed of low-bit attention, grouped by GPU.}
  Fidelity is against the BF16 FlashAttention-4 output of the same seed, and VBench scores are in \cref{sec:appendix-protocol}. Attn.\ is the MiniMax-H3 attention speedup over BF16 FlashAttention-4.
  Bold and underline mark the best and second best within a block. $^\star$:~deployed. $^\dagger$:~training-free. NaN: the run produced non-finite outputs and aborted (\cref{sec:appendix-directp}).}
  \label{tab:main}
  {\setlength{\tabcolsep}{1pt}\footnotesize\renewcommand{\arraystretch}{1.15}
\begin{tabularx}{\linewidth}{c l c c YYY YYY YYY YYY}
\toprule
 & & & & \multicolumn{3}{c}{Wan2.2-T2V-A14B} & \multicolumn{3}{c}{LongCat-Video} & \multicolumn{3}{c}{HunyuanVideo-1.5} & \multicolumn{3}{c}{MiniMax-H3 distilled} \\
\cmidrule(lr){5-7}\cmidrule(lr){8-10}\cmidrule(lr){11-13}\cmidrule(lr){14-16}
GPU & Method & QK/PV & Attn. & {\scriptsize PSNR$\uparrow$} & {\scriptsize SSIM$\uparrow$} & {\scriptsize LPIPS$\downarrow$} & {\scriptsize PSNR$\uparrow$} & {\scriptsize SSIM$\uparrow$} & {\scriptsize LPIPS$\downarrow$} & {\scriptsize PSNR$\uparrow$} & {\scriptsize SSIM$\uparrow$} & {\scriptsize LPIPS$\downarrow$} & {\scriptsize PSNR$\uparrow$} & {\scriptsize SSIM$\uparrow$} & {\scriptsize LPIPS$\downarrow$} \\
\midrule
 & FlashAttention-4 & BF16 & 1.00\texttimes & \na & \na & \na & \na & \na & \na & \na & \na & \na & \na & \na & \na \\
\midrule
 & SageAttention2 & 8/8 & \na & 20.3 & 0.730 & 0.206 & 23.1 & 0.793 & 0.122 & 15.6 & 0.581 & 0.382 & 19.9 & 0.723 & 0.252 \\
 & FlashAttention-4 (8-bit) & 8/8 & 1.24\texttimes & 18.8 & 0.686 & 0.248 & 21.3 & 0.753 & 0.154 & 15.9 & 0.592 & 0.370 & 20.2 & 0.731 & 0.243 \\
 & \quad + QK Hadamard & 8/8 & 1.24\texttimes & 18.8 & 0.685 & 0.253 & 21.6 & 0.762 & 0.148 & 16.0 & 0.593 & 0.367 & \underline{20.5} & \underline{0.737} & \underline{0.233} \\
 & Attn-QAT$^\dagger$ & 4/8 & 1.40\texttimes & 15.5 & 0.561 & 0.428 & 16.4 & 0.556 & 0.398 & 12.2 & 0.459 & 0.562 & 15.0 & 0.581 & 0.467 \\
\cmidrule{2-16}
\oursg & \textbf{Ours (V-Smooth)} & 8/8 & 1.21\texttimes & \textbf{22.6} & \textbf{0.792} & \textbf{0.152} & \textbf{24.5} & \textbf{0.822} & \textbf{0.102} & \textbf{18.4} & \textbf{0.675} & \textbf{0.271} & \textbf{21.0} & \textbf{0.751} & \textbf{0.218} \\
\oursg \multirow{-6}{*}{\makecell{B200\\8-bit}} & \textbf{Ours (V-Smooth + ExpCast)}$^\star$ & 8/8 & \textbf{1.59\texttimes} & \underline{20.5} & \underline{0.744} & \underline{0.191} & \underline{23.8} & \underline{0.807} & \underline{0.109} & \underline{17.5} & \underline{0.648} & \underline{0.301} & 20.2 & 0.726 & 0.248 \\
\midrule
 & SageAttention3 & 4/4 & 2.30\texttimes & 13.7 & 0.501 & 0.472 & 15.1 & 0.526 & 0.382 & 10.6 & 0.417 & 0.636 & 16.1 & 0.587 & 0.408 \\
 & \quad + QK Hadamard & 4/4 & 2.30\texttimes & 14.0 & 0.512 & 0.451 & 15.1 & 0.530 & 0.375 & 10.6 & 0.417 & 0.631 & 16.0 & 0.588 & 0.403 \\
\cmidrule{2-16}
\oursg & \textbf{Ours (V-Smooth)} & 4/4 & 2.85\texttimes & \textbf{16.8} & \textbf{0.614} & \textbf{0.320} & \textbf{19.6} & \textbf{0.692} & \textbf{0.202} & \textbf{14.6} & \textbf{0.560} & \textbf{0.437} & \textbf{16.8} & \textbf{0.607} & \textbf{0.377} \\
\oursg \multirow{-4}{*}{\makecell{RTX PRO\\6000\\4-bit}} & \textbf{Ours (V-Smooth + ExpCast)}$^\star$ & 4/4 & \textbf{3.67\texttimes} & \underline{16.3} & \underline{0.600} & \underline{0.340} & \underline{19.1} & \underline{0.674} & \underline{0.218} & \underline{14.4} & \underline{0.554} & \underline{0.446} & \textbf{16.8} & \textbf{0.607} & \underline{0.378} \\
\midrule
 & Attn-QAT$^\dagger$ & 4/4 & 1.36\texttimes & 13.8 & 0.502 & 0.502 & 14.9 & 0.533 & 0.394 & 9.6 & 0.370 & 0.669 & 14.5 & 0.573 & 0.478 \\
 & \quad + QK Hadamard & 4/4 & 1.36\texttimes & 13.6 & 0.500 & 0.493 & 15.4 & 0.551 & 0.366 & 9.7 & 0.370 & 0.661 & 15.0 & 0.587 & 0.456 \\
 & Direct-P (accurate) & 4/4 & 1.56\texttimes & \nanc & \nanc & \nanc & 13.5 & 0.485 & 0.456 & 9.8 & 0.376 & 0.661 & \nanc & \nanc & \nanc \\
 & Direct-P (fast) & 4/4 & 1.81\texttimes & \nanc & \nanc & \nanc & \nanc & \nanc & \nanc & \nanc & \nanc & \nanc & \nanc & \nanc & \nanc \\
\cmidrule{2-16}
\oursg & \textbf{Ours (V-Smooth)} & 4/4 & 1.32\texttimes & \textbf{17.3} & \textbf{0.635} & \textbf{0.302} & \textbf{19.9} & \textbf{0.705} & \textbf{0.192} & \textbf{13.8} & \textbf{0.516} & \textbf{0.468} & \textbf{16.7} & \textbf{0.636} & \textbf{0.373} \\
\oursg \multirow{-6}{*}{\makecell{B300\\4-bit}} & \textbf{Ours (V-Smooth + ExpCast)}$^\star$ & 4/4 & \textbf{2.07\texttimes} & \underline{16.5} & \underline{0.609} & \underline{0.334} & \underline{19.5} & \underline{0.689} & \underline{0.206} & \underline{13.1} & \underline{0.493} & \underline{0.497} & \underline{16.3} & \underline{0.624} & \underline{0.390} \\
\bottomrule
\end{tabularx}}

\end{table*}

\subsection{Setup}
\label{sec:exp-setup}

\myparagraph{Models.}
We evaluate \method on four open-weight video DiTs: Wan2.2-T2V-A14B~\cite{wan2025wan} at 720p, LongCat-Video~\cite{team2025longcat} at 480p, HunyuanVideo-1.5~\cite{wu2025hunyuanvideo} at 720p, and MiniMax-H3~\cite{minimax2026h3} at 1344\texttimes768 with the 8-step acceleration LoRA of Alibaba PAI~\cite{alibabapai2026h3acc}.
Every method generates \NumPrompts videos per model from MovieGen Bench prompts~\cite{polyak2024moviegen}.

\myparagraph{Metrics.}
Following prior work on efficient video attention~\cite{xi2025sparse,li2026pisa,li2026deltaquant,xi2026quantvideogen}, we report PSNR, SSIM~\cite{wang2004image}, LPIPS~\cite{zhang2018unreasonable}, and two VBench dimensions~\cite{huang2024vbench}.
Subject consistency measures whether the subject keeps its appearance across frames, and \VBenchDimB measures per-frame distortion such as blur and noise.
We measure the attention speedup on one MiniMax-H3 attention call without preprocessing, and the denoising speedup on its denoising loop with preprocessing but without text encoding or VAE decode.

\myparagraph{Baselines.}
Every fidelity number is scored against the BF16 FlashAttention-4~\cite{zadouri2026flashattention4} output of the same model and seed, and every speedup is normalized to that kernel on the same card.
On the H200, the FlashAttention-4 interface runs its SM90 kernel.
At 8 bits we compare against SageAttention2~\cite{zhang2024sageattention2} in its released INT8/FP8 configuration and against an 8-bit FlashAttention-4 that rounds both operands to the nearest code per block.
We also run Attn-QAT~\cite{zhang2026attnqat} training-free, while its full method retrains each model.
SageAttention2 ships no B200 kernel.
We recompile it for B200 (\cref{sec:appendix-protocol}) and also report its native speed on H200.
At 4 bits we compare against SageAttention3~\cite{zhang2025sageattention3} on the RTX PRO 6000 and RTX 5090, since its kernel does not run on the B300, and against Attn-QAT and the released fast and accurate policies of Direct-P~\cite{hu2026hardware} on the B300.
The 8-bit FlashAttention-4, SageAttention3, and the 4-bit Attn-QAT each also run with a fused Hadamard rotation of $\mathbf Q$ and $\mathbf K$~\cite{dao2023fasthadamard}, the strongest training-free fix for $\QK$.

\myparagraph{Implementation.}
We implement \method in CuTe.
At 8 bits we add V-Smooth and ExpCast-FP8 to FlashAttention-4 in place on the B200 and H200.
At 4 bits we write our own NVFP4 kernel with V-Smooth and ExpCast-FP4 for the RTX PRO 6000, RTX 5090, and B300.
\Cref{sec:appendix-protocol} gives the evaluation details, and \cref{sec:appendix-fusion} the preprocessing fusion.

\subsection{Main Results}
\label{sec:exp-main}

\myparagraph{\method matches the best baseline at 8 bits and leads at 4 bits.}
At 8 bits (\cref{tab:main}), V-Smooth alone leads on PSNR, SSIM, and LPIPS for every model, with \LPIPSGainEightBit lower LPIPS than SageAttention2.
Its PSNR gain over SageAttention2 is 2.3~dB on Wan2.2 and 2.8~dB on HunyuanVideo-1.5.
With ExpCast-FP8 added for speed (\cref{fig:efficiency}), the deployed kernel gives up \ExpCastFidelityCost of this PSNR and stays at or above SageAttention2 on all four models.
Without its retraining, Attn-QAT scores \AttnQATTrainingFreeGap lower PSNR than SageAttention2.
At 4 bits on the RTX PRO 6000, the deployed kernel improves PSNR over SageAttention3, with or without the Hadamard rotation, by \FourBitGainRange on every model, while ExpCast-FP4 loses at most \ExpCastFourBitCost to the exact exponential.
On the B300, the deployed kernel leads training-free Attn-QAT by \BThreeHundredGainAttnQAT and Direct-P by \BThreeHundredGainDirectP of PSNR on every model where they finish, and the fast policy of Direct-P aborts on non-finite outputs on all four models.
On the two VBench dimensions, every 8-bit method except Attn-QAT stays within 0.01 of BF16, and our 4-bit kernels within \VBenchSpreadFourBit on both GPUs.
At 4 bits, the VBench scores of Attn-QAT and Direct-P drop below BF16 by up to 0.04 and 0.05 on HunyuanVideo-1.5.
Per attention call, the deployed kernel leads the strongest baseline by about 1~dB of SNR at both precisions (\cref{sec:appendix-kernel-snr}).
Please refer to \cref{fig:qual-main} and \cref{sec:appendix-qualitative} for visual results.

\begin{figure*}[t]
  \centering
  \Asset[\linewidth]{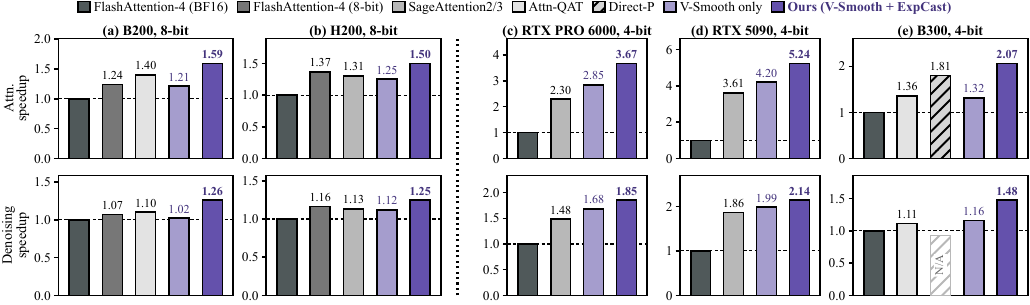}{2.00in}
  \caption{\textbf{\method converts low-bit peak throughput into speedup on
  every card we evaluate.} Speedup over BF16 FlashAttention-4 (dashed) on MiniMax-H3, for the attention kernel (top) and the denoising loop with preprocessing (bottom). The denoising speedup is smaller because attention is only part of each step.}
  \label{fig:efficiency}
\end{figure*}

\myparagraph{Rotating $\QK$ adds little.}
Moreover, fusing the Hadamard rotation into FlashAttention-4 at 8 bits or SageAttention3 at 4 bits shifts PSNR by at most \HadamardPSNRGain on any model, and into Attn-QAT on the B300 by at most \HadamardPSNRGainAttnQAT.
The remaining error is therefore in $\PV$, which V-Smooth targets (\cref{sec:method-value}).

\subsection{Efficiency}
\label{sec:exp-efficiency}
\myparagraph{8 bits.}
At default clocks on B200 and H200, \method computes attention \AttnSpeedupBTwoHundred and \AttnSpeedupHTwoHundred faster than BF16 FlashAttention-4.
The 8-bit \mbox{FlashAttention-4} baseline, which still computes the exponential on the special function units, reaches only \AttnSpeedupNaiveBTwoHundred and \AttnSpeedupNaiveHTwoHundred.
The denoising loop runs \DenoiseSpeedupDatacenterRange faster on both cards.
With the B200 clock locked at \LockedClockBTwoHundred, this baseline is slower than BF16 (\LockedNaiveSpeedup), while \method stays \LockedSpeedupBTwoHundred faster (\cref{sec:appendix-clock}).

\myparagraph{4 bits.}
On the RTX PRO 6000 and RTX 5090, \method computes attention \AttnSpeedupRTXPro and \AttnSpeedupRTXFiftyNinety faster than BF16 FlashAttention-4, against \SageThreeAttnPro and \SageThreeAttnFiftyNinety for SageAttention3, and runs the denoising loop \DenoiseSpeedupRTXPro and \DenoiseSpeedupRTXFiftyNinety faster, against \SageThreeDenoisePro and \SageThreeDenoiseFiftyNinety.
On the B300, \method runs attention \AttnSpeedupBThreeHundredFour faster and the denoising loop \DenoiseSpeedupBThreeHundredFour faster, against \AttnQATAttnBThreeHundred and \AttnQATDenoiseBThreeHundred for Attn-QAT and \DirectPFastAttnBThreeHundred for Direct-P, whose end-to-end run produces non-finite outputs (\cref{sec:appendix-directp}).

\myparagraph{Preprocessing.}
On Wan2.2 on B200, we fuse the preprocessing before each attention call into one chain that takes \PrepFusedCall per call, \PrepFusionTotal less than the unfused chain (\cref{sec:appendix-fusion}).
On the RTX PRO 6000, our 4-bit preprocessing takes \OursPrepWanPro per call on Wan2.2, against \SageThreePrepWanPro for SageAttention3.

\subsection{Ablation Study}
\label{sec:exp-ablation}
\begin{table}[t]
  \centering
  \caption{\textbf{V-Smooth adds accuracy, and ExpCast adds speed.}
  We run LongCat-Video on the RTX PRO 6000 with 32 of the 100 prompts. Attn.\ is the speedup over BF16 FlashAttention-4. SageAttention3 fuses the QK Hadamard rotation. Ours runs V-Smooth on the first 25\% of steps.}
  \label{tab:ablation}
  {\setlength{\tabcolsep}{2.6pt}\footnotesize\renewcommand{\arraystretch}{1.15}
\begin{tabularx}{\linewidth}{lccYYYc}
\toprule
Kernel & V-Smooth & ExpCast & PSNR\,$\uparrow$ & SSIM\,$\uparrow$ & LPIPS\,$\downarrow$ & Attn.\,$\uparrow$ \\
\midrule
SageAttention3 & & & 15.78 & 0.563 & 0.367 & 2.28\texttimes \\
\midrule
\multirow{4}{*}{Ours} & & & 18.60 & 0.651 & 0.262 & 2.87\texttimes \\
 & \checkmark & & \textbf{20.45} & \textbf{0.715} & \textbf{0.198} & 2.84\texttimes \\
 & & \checkmark & 17.88 & 0.627 & 0.283 & \textbf{3.71\texttimes} \\
\oursg & \checkmark & \checkmark & 19.93 & 0.696 & 0.214 & 3.68\texttimes \\
\bottomrule
\end{tabularx}}

\end{table}
\myparagraph{V-Smooth adds accuracy, and ExpCast adds speed.}
Without V-Smooth and ExpCast, our 4-bit kernel, which divides by the sum of the quantized probabilities (\cref{sec:method-value}), already leads SageAttention3 by \KernelOnlyGainLongCat on LongCat-Video (\cref{tab:ablation}).
Adding V-Smooth raises PSNR by about 2~dB, whether the exponential is exact or computed by ExpCast.
ExpCast makes attention \ExpCastAblGain faster and costs \ExpCastCostAbl of PSNR.
With both, the kernel is still 1.3~dB above the kernel without either (19.93 against 18.60~dB, \cref{tab:ablation}).

\begin{figure}[t]
  \centering
  \Asset[\linewidth]{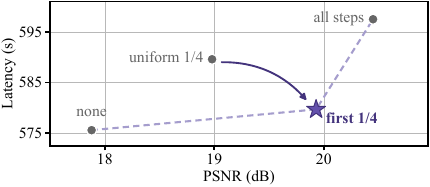}{1.25in}
  \caption{\textbf{Running V-Smooth on the first quarter of the steps buys most of the accuracy for little time.}
  The first quarter adds 2.1~dB for 4~s per clip, and the remaining steps add only 0.5~dB for 18~s more. Spreading the same number of steps evenly is both less accurate and slower. Setting of \cref{tab:ablation}, with ExpCast-FP4.}
  \label{fig:placement}
\end{figure}
\myparagraph{V-Smooth matters most in the first denoising steps.}
In our configuration, V-Smooth runs only on the first \GroupingStepFraction of the steps, as the early steps are more sensitive to attention error~\cite{xi2025sparse,li2025radial}.
These early steps beat the same number of steps spread evenly by \PlacementFirstOverUniform of PSNR (\cref{fig:placement}).
Applying V-Smooth on every step adds \PlacementAllOverFirst more but makes each clip 17.8~s slower.
In the 8-bit kernel the same two gaps are \PlacementFirstOverUniformEightBit and \PlacementAllOverFirstEightBit (\cref{sec:appendix-results}).

\section{Conclusion}
\label{sec:conclusion}

We have presented \method, a training-free low-bit attention kernel for video DiTs that lowers the value error and removes the exponential.
V-Smooth sorts similar value tokens into the same blocks, subtracts the mean of each block, and quantizes only the residual.
ExpCast writes each probability code directly from the score, without the exponential.
On four video DiTs, it matches the strongest baseline at 8 bits and is the most accurate training-free attention at 4 bits.
It runs attention up to \AttnSpeedupBTwoHundred faster than BF16 FlashAttention-4 at 8 bits and up to \AttnSpeedupRTXFiftyNinety faster at 4 bits.
At 4 bits, it leads SageAttention3 by \FourBitGainRange of PSNR and makes the denoising loop up to \DenoiseSpeedupRTXFiftyNinety faster.
\raggedbottom

\clearpage
{
    \small
    \bibliographystyle{nunchux}
    \bibliography{main}
}

\clearpage
\setlength{\textfloatsep}{6pt plus 2pt minus 4pt}
\setlength{\dbltextfloatsep}{6pt plus 2pt minus 4pt}
\raggedbottom
\appendix
\crefalias{section}{appendix}
\crefalias{subsection}{appendix}
\section{Experimental Setup}
\label{sec:appendix-protocol}

\myparagraph{Generation.}
Wan2.2 generates 81 frames at 720p, LongCat-Video 93 frames at 480\texttimes832, HunyuanVideo-1.5 121 frames at 720p, and MiniMax-H3 243 frames at 1344\texttimes768 with synchronized audio.
The four settings of \cref{tab:main} thus run 75.6K, 37.4K, 111.7K, and 73.5K attention tokens.
For MiniMax-H3 the count includes the video, audio, and text tokens it packs into one sequence.
Following Quant VideoGen~\cite{xi2026quantvideogen}, we use the extended MovieGen Bench prompts released with Self-Forcing~\cite{huang2025selfforcing}.
All methods of a setting share the same prompts and random seeds.
Wan2.2 uses 40 denoising steps, and LongCat-Video and HunyuanVideo-1.5 use 50.
MiniMax-H3 runs with the 8-step Parallel Decoding Distillation LoRA released by Alibaba PAI~\cite{shaul2026pdd,alibabapai2026h3acc}, with the Euler solver and no classifier-free guidance.

\myparagraph{Metrics.}
PSNR, SSIM, and LPIPS are computed against the BF16 FlashAttention-4 output of the same model at the same seed, rendered on the same type of GPU.
Following Sparse VideoGen, PISA, DeltaQuant, and Quant VideoGen~\cite{xi2025sparse,li2026pisa,li2026deltaquant,xi2026quantvideogen}, the reference is the unmodified model.
The three metrics thus measure how far a kernel moves the output of the model.
VBench scores the quality of the clip itself (\cref{tab:vbench}).

\begin{table*}[t]
  \centering
  \caption{\textbf{VBench scores of \cref{tab:main}.} S.C.\ and I.Q.\ are the VBench dimensions \VBenchDimA and \VBenchDimB, scored on the same clips and rows as \cref{tab:main}.
  $^\star$:~deployed. $^\dagger$:~training-free. NaN: the run produced non-finite outputs and aborted (\cref{sec:appendix-directp}).}
  \label{tab:vbench}
  {\setlength{\tabcolsep}{2.4pt}\footnotesize\renewcommand{\arraystretch}{0.98}
\begin{tabularx}{\linewidth}{c l c YY YY YY YY}
\toprule
 & & & \multicolumn{2}{c}{Wan2.2-T2V-A14B} & \multicolumn{2}{c}{LongCat-Video} & \multicolumn{2}{c}{HunyuanVideo-1.5} & \multicolumn{2}{c}{MiniMax-H3 distilled} \\
\cmidrule(lr){4-5}\cmidrule(lr){6-7}\cmidrule(lr){8-9}\cmidrule(lr){10-11}
GPU & Method & QK/PV & S.C.\,$\uparrow$ & I.Q.\,$\uparrow$ & S.C.\,$\uparrow$ & I.Q.\,$\uparrow$ & S.C.\,$\uparrow$ & I.Q.\,$\uparrow$ & S.C.\,$\uparrow$ & I.Q.\,$\uparrow$ \\
\midrule
 & FlashAttention-4 & BF16 & 0.932 & 0.703 & 0.941 & 0.694 & 0.933 & 0.680 & 0.899 & 0.668 \\
\midrule
 & SageAttention2 & 8/8 & 0.931 & 0.703 & 0.941 & 0.697 & 0.931 & 0.676 & 0.897 & 0.671 \\
 & FlashAttention-4 (8-bit) & 8/8 & 0.929 & 0.704 & 0.940 & 0.697 & 0.932 & 0.676 & 0.899 & 0.668 \\
 & \quad + QK Hadamard & 8/8 & 0.929 & 0.706 & 0.940 & 0.695 & 0.932 & 0.678 & 0.898 & 0.669 \\
 & Attn-QAT$^\dagger$ & 4/8 & 0.913 & 0.699 & 0.908 & 0.656 & 0.911 & 0.662 & 0.896 & 0.643 \\
\cmidrule{2-11}
\oursg & \textbf{Ours (V-Smooth)} & 8/8 & 0.931 & 0.704 & 0.941 & 0.696 & 0.932 & 0.679 & 0.899 & 0.669 \\
\oursg \multirow{-6}{*}{\makecell{B200\\8-bit}} & \textbf{Ours (V-Smooth + ExpCast)}$^\star$ & 8/8 & 0.932 & 0.705 & 0.941 & 0.696 & 0.931 & 0.677 & 0.897 & 0.669 \\
\midrule
 & SageAttention3 & 4/4 & 0.928 & 0.701 & 0.938 & 0.692 & 0.923 & 0.670 & 0.907 & 0.687 \\
 & \quad + QK Hadamard & 4/4 & 0.931 & 0.706 & 0.938 & 0.695 & 0.923 & 0.672 & 0.908 & 0.686 \\
\cmidrule{2-11}
\oursg & \textbf{Ours (V-Smooth)} & 4/4 & 0.932 & 0.706 & 0.942 & 0.696 & 0.928 & 0.678 & 0.907 & 0.685 \\
\oursg \multirow{-4}{*}{\makecell{RTX PRO\\6000\\4-bit}} & \textbf{Ours (V-Smooth + ExpCast)}$^\star$ & 4/4 & 0.932 & 0.704 & 0.941 & 0.698 & 0.929 & 0.677 & 0.905 & 0.683 \\
\midrule
 & Attn-QAT$^\dagger$ & 4/4 & 0.915 & 0.695 & 0.939 & 0.694 & 0.896 & 0.638 & 0.900 & 0.651 \\
 & \quad + QK Hadamard & 4/4 & 0.927 & 0.701 & 0.938 & 0.693 & 0.912 & 0.652 & 0.899 & 0.647 \\
 & Direct-P (accurate) & 4/4 & \nanc & \nanc & 0.929 & 0.670 & 0.900 & 0.626 & \nanc & \nanc \\
 & Direct-P (fast) & 4/4 & \nanc & \nanc & \nanc & \nanc & \nanc & \nanc & \nanc & \nanc \\
\cmidrule{2-11}
\oursg & \textbf{Ours (V-Smooth)} & 4/4 & 0.931 & 0.704 & 0.941 & 0.695 & 0.927 & 0.670 & 0.896 & 0.655 \\
\oursg \multirow{-6}{*}{\makecell{B300\\4-bit}} & \textbf{Ours (V-Smooth + ExpCast)}$^\star$ & 4/4 & 0.929 & 0.702 & 0.942 & 0.695 & 0.925 & 0.665 & 0.897 & 0.655 \\
\bottomrule
\end{tabularx}}

\end{table*}

\myparagraph{Timing.}
Attention speedup is the attention time of FlashAttention-4 divided by that of the evaluated method.
For V-Smooth, it averages the kernel inside the grouping window and the kernel outside it, weighted by their share of the steps (\cref{tab:schedule-cost}).
Both kernels are timed as isolated launches on the captured post-RoPE tensors of each setting.
CUPTI records a ten-second window after twenty seconds of warm-up, and we report the per-launch median.
Quantization, token grouping, and the rest of the preprocessing run outside this timer and count only toward the denoising speedup.
At 8 bits, \method uses the fused preprocessing of \cref{sec:appendix-fusion}.
At 4 bits, its preprocessing is fused into three passes (\cref{sec:appendix-fusion}), and the baselines use their own unfused preprocessing at both precisions.
Denoising speedup is the wall clock of the DiT denoising loop, CUDA-synchronized around the call.
It excludes text encoding and VAE decode, whose cost does not depend on the attention kernel.
All latencies are measured on one exclusive GPU after warm-up.
Fidelity runs share GPUs and are not timed.

\myparagraph{Baseline kernels.}
We use FlashAttention-4~\cite{zadouri2026flashattention4} as the BF16 reference because it builds and runs on every card we report.
On the H200 it dispatches to its Hopper path, the FlashAttention-3 algorithm~\cite{shah2024flashattention}.
SageAttention2 ships no B200 kernel, so we recompile its kernel for \texttt{sm\_100a} to measure its B200 fidelity.
The recompiled kernel runs 0.29\texttimes\ as fast as BF16 FlashAttention-4, which \cref{fig:teaser} shows as $<$1\texttimes.
We build the 8-bit FlashAttention-4 from its BF16 kernel: it quantizes $\mathbf Q$ and $\mathbf K$ to INT8 and $\mathbf P$ and $\mathbf V$ to E4M3, and keeps the exponential of the BF16 kernel.
It divides by the sum of the probabilities before they are quantized, as do SageAttention2, SageAttention3, and Attn-QAT, while Direct-P and the deployed \method kernels divide by the sum of the quantized probabilities.
Following SageAttention3, our 4-bit kernel also subtracts the mean of each tile of 128 queries (192 on the RTX PRO 6000) and adds its product with $\mathbf K^\top$ back to the scores.

\myparagraph{Scope.}
\method targets the bidirectional self-attention of video DiTs.
V-Smooth permutes the keys and values, which leaves the output unchanged only when no attention mask is applied.
Causal or block-causal models would therefore need the permutation within each mask block, which we leave to future work.
We also evaluate only head dimension 128, which all four models use.
Finally, ExpCast-FP8 is least accurate on the rare rows whose probabilities fall mostly below the E4M3 normal range (\cref{sec:appendix-fp8}).

\section{Implementation Details}
\label{sec:appendix-impl}

\myparagraph{V-Smooth schedule.}
V-Smooth groups on the first $\lceil T/4\rceil$ of $T$ denoising steps, on every layer of those steps.
The permutation is recomputed every \PermutationReuseSteps steps and reused in between, and each recomputation warm-starts $k$-means from the previous centroids.
A 40-step Wan2.2 run thus recomputes it at steps 0, 4, and 8, and a 50-step LongCat-Video or HunyuanVideo-1.5 run at steps 0, 4, 8, and 12.
The 8-step MiniMax-H3 setting groups on its first two steps.
The clustering is plain Lloyd $k$-means with $k=N/(4B_v)$, 147 on Wan2.2, with the same settings on all four models.
It runs 10 iterations from a cold start, and 2 at 8 bits or 10 at 4 bits from a warm start.
\Cref{tab:schedule-cost} reports the cost of the two mechanisms that keep this window cheap.
At 8 bits, a warm start costs \KMeansWarmCall against \KMeansColdCall cold on the B200.
Reusing the permutation removes the clustering from the step and leaves only the gather.
The V-Smooth-off row is the ExpCast-FP8 kernel run outside the window.

\begin{table}[t]
  \centering
  \caption{\textbf{The deployed schedule adds \GroupingOverhead per attention call over V-Smooth off.} One Wan2.2 720p
  self-attention call on one B200, preprocessing and attention kernel
  together, $N = 75600$, $H = 40$, $d = 128$. The preprocessing column is the
  whole chain and contains the $k$-means column. The first four rows are one
  fixed configuration each. $^\ast$The deployed row is the
  duty cycle of a 40-step clip: the window is the first ten steps, the
  permutation is recomputed at steps 0, 4 and 8, and V-Smooth is off on the
  remaining thirty. Applying V-Smooth on all 40 steps instead, one cold and nine warm
  recomputations, averages \GroupingEveryStepCall, \GroupingOverheadEveryStep
  over the V-Smooth-off row.}
  \label{tab:schedule-cost}
  {\setlength{\tabcolsep}{3pt}\footnotesize
\begin{tabularx}{\linewidth}{lYYYY}
\toprule
Per call (ms) & $k$-means & Prep. & Kernel & Total \\
\midrule
Cold $k$-means & 5.85 & 8.78 & 58.39 & 67.17 \\
Warm $k$-means & 2.33 & 5.26 & 58.39 & 63.65 \\
Reused permutation & \na & 2.93 & 58.39 & 61.32 \\
V-Smooth off & \na & 2.55 & 52.09 & 54.65 \\
\midrule
\ours \textbf{Deployed}$^\ast$ & \textbf{0.26} & \textbf{2.91} & \textbf{53.66} & \textbf{56.58} \\
\bottomrule
\end{tabularx}}

\end{table}

\myparagraph{ExpCast-FP4 encoding.}
The 4-bit kernel scales the probabilities as $\mathbf P=2^{z+7}$, where $z$ is the base-2 score measured from the running reference and is at most 2.5.
The largest values of a row thus reach the top of E2M1 after the group scale.
One FFMA and one \texttt{cvt.rn.relu.f16x2.f32} turn each score into $t=\operatorname{Round}(32(z+15))$, so scores more than 15 octaves below the reference become zero.
Each group of 16 takes the scale byte $b=\lceil t_{\max}/4\rceil-28$, with $t_{\max}$ clamped below at 144, rounded up so that no element exceeds 6, the E2M1 maximum.
The half-precision bit pattern $32t+14336-128b$ then encodes the element over its group scale, and \texttt{cvt.rn.satfinite.e2m1x2.f16x2} writes the E2M1 pair.
The reference moves by whole octaves, and only when the row maximum grows by more than 2.5 octaves, so the factor $\alpha_i$ that rescales the accumulator is an exact power of two.

\myparagraph{Fused preprocessing.}
\phantomsection\label{sec:appendix-fusion}
Ahead of each 8-bit attention call, a chain of memory-bound passes produces the low-bit operands.
The chain applies rotary embedding, computes the key channel mean and the value block means, gathers by $\pi$, rotates $\mathbf Q$ and $\mathbf K$, and quantizes.
Run eagerly, every pass writes a full high-precision tensor back to HBM for the next one to read.
\method fuses each pass into the quantizer, starting from the last one, until one pass covers the chain (\cref{fig:prep-fusion}).
Fusing the Hadamard rotation, then the smoothing means and the gather, and then the rotary embedding leaves one pass.
Reimplementing the remaining passes in CuTe makes the chain \PrepFusionTotal faster end to end.
The whole chain then costs \PrepFusedCall per attention call, of which grouping takes \PrepGroupingShare.
At 4 bits on the RTX PRO 6000, the same fusion leaves three passes, over $\mathbf Q$, over $\mathbf K$, and over $\mathbf K$ and $\mathbf V$.
They take \OursPrepWanPro per call on Wan2.2 and \OursPrepHThree on MiniMax-H3, against \SageThreePrepWanPro and \SageThreePrepHThree for SageAttention3, and a cold $k$-means adds 15.7~ms on Wan2.2.

\begin{figure}[t]
  \centering
  \includegraphics[width=\linewidth]{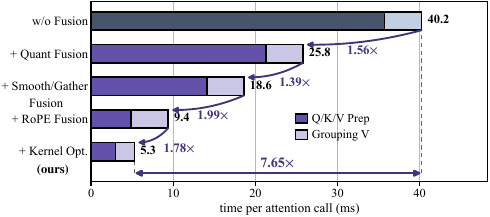}
  \caption{\textbf{Fusion makes preprocessing \PrepFusionTotal faster.} One V-Smooth
  attention call on Wan2.2, one B200. Stages are cumulative. A DiT forward makes
  40 calls: \PrepUnfusedForward unfused, \PrepFusedForward fused.}
  \label{fig:prep-fusion}
\end{figure}

\section{ExpCast: Proof and Accuracy}
\label{sec:appendix-fp8}

\Cref{eq:direct-fp8-row-bound} is the first inequality of the following proposition, with $D=\operatorname{diam}_2(\mathbf V')=\max_{k,l}\|\mathbf v'_k-\mathbf v'_l\|_2$, the largest distance between two value rows.
\begin{proposition}[Consistent-denominator bound]
\label{prop:direct-fp8-row}
Let $\mathbf p$ be an exact softmax row, and let the kernel write $\widehat w_k=\rho_k w_k$ for each unnormalized entry $w_k$, with $\rho_k\in[a,b]$ and $a>0$.
If it normalizes by $\sum_k\widehat w_k$ to obtain $\widehat{\mathbf p}$, then for $\mathbf o=\mathbf p^\top\mathbf V'$ and $\widehat{\mathbf o}=\widehat{\mathbf p}^{\top}\mathbf V'$,
\begin{align}
  \|\widehat{\mathbf o}-\mathbf o\|_2
  &\leq\operatorname{TV}(\mathbf p,\widehat{\mathbf p})\,\operatorname{diam}_2(\mathbf V'),
  \nonumber\\
  \operatorname{TV}(\mathbf p,\widehat{\mathbf p})
  &\leq(\sqrt b-\sqrt a)\,/\,(\sqrt b+\sqrt a).
  \label{eq:appendix-prop-bound}
\end{align}
\end{proposition}
The proposition holds for any way of writing the probabilities, as long as each written entry stays within a fixed ratio of the exact one and the kernel divides by the sum of what it wrote.
We prove it first and then derive the ratios $a$ and $b$ of ExpCast-FP8.

\begin{proof}
For the first inequality, let $\mathbf d=\widehat{\mathbf p}-\mathbf p$.
Both rows sum to one, so the positive and the negative parts of $\mathbf d$ each sum to $t=\operatorname{TV}(\mathbf p,\widehat{\mathbf p})$, and
\begin{equation}
  \widehat{\mathbf o}-\mathbf o
  =\frac1t\sum_{k,l} d_k^+\,d_l^-\,(\mathbf v'_k-\mathbf v'_l).
  \label{eq:appendix-coupling}
\end{equation}
Each difference $\mathbf v'_k-\mathbf v'_l$ has norm at most $\operatorname{diam}_2(\mathbf V')$, and the weights $d_k^+d_l^-/t$ sum to $t$, which gives the first inequality.
A component shared by all value rows cancels in every difference, so it causes no error.

For the second inequality, normalization gives $\widehat p_k=p_k \rho_k/\bar\rho$ with $\bar\rho=\mathbb{E}_{\mathbf p}[\rho]$, so
$\operatorname{TV}(\mathbf p,\widehat{\mathbf p})=\tfrac12\sum_k p_k|\rho_k-\bar\rho|/\bar\rho$.
This is a convex function of $\boldsymbol\rho$ divided by a positive linear one, hence quasiconvex, so its maximum over the box $[a,b]^N$ is attained at a vertex, where every $\rho_k$ is $a$ or $b$.
If a fraction $q$ of the $\mathbf p$-mass receives $b$ and the rest receives $a$, then
\begin{equation}
  \operatorname{TV}(\mathbf p,\widehat{\mathbf p})
  =\frac{q(1-q)(b-a)}{a+q(b-a)}.
  \label{eq:appendix-fp8-tv-q}
\end{equation}
The maximizer is $q=\sqrt{a}/(\sqrt a+\sqrt b)$, which yields
\begin{equation}
  \operatorname{TV}(\mathbf p,\widehat{\mathbf p})
  \leq\frac{\sqrt b-\sqrt a}{\sqrt b+\sqrt a}.
  \label{eq:appendix-fp8-tv}
\end{equation}
\end{proof}

\myparagraph{The ratios of ExpCast-FP8.}
ExpCast-FP8 writes the code of \cref{eq:direct-fp8} for the value $2^{z_k+8}$, where $w_k=2^{z_k}$.
The factor $2^8$ is the same for every entry of a row and cancels in the normalization, so we compare the decoded value $\widehat w_k$ with $2^8w_k$ and write $\rho_k=\widehat w_k/(2^8w_k)$.
Let $f_k=z_k-\lfloor z_k\rfloor\in[0,1)$ be the fractional part of $z_k$.
For a normal entry, the exact mantissa coordinate is $M=8(2^{f_k}-1)$, while the linear rule uses $8f_k$.
Their difference is
\begin{equation}
  \varepsilon=8\bigl[(2^f-1)-f\bigr],
  \label{eq:appendix-fp8-code-error}
\end{equation}
with range $[-0.6886,0]$.
Up to sign, it is Mitchell's straight-line approximation error of $\log_2$, scaled by the eight codes a doubling spans~\cite{mitchell1962computer}.
The linear rule is therefore too large by up to 0.6886 codes, and the constant shift that centers this error is $-0.6886/2=-0.3443$, and ExpCast uses $\beta=-0.35$.
The resulting ratio is
\begin{equation}
  \rho(f_k)
  =\frac{1+\operatorname{Round}(8f_k-0.35)/8}{2^{f_k}}.
  \label{eq:appendix-fp8-ratio-function}
\end{equation}
The ratio is monotone within each rounding cell.
Evaluating both one-sided limits at the rounding boundaries $f=(j+1/2+0.35)/8$ for $j=0,\dots,7$, where $j=7$ carries into the next power of two, gives
\begin{equation}
  0.9290\leq \rho(f)\leq1.0745.
  \label{eq:appendix-fp8-ratio}
\end{equation}
With these ratios, \cref{prop:direct-fp8-row} gives $\operatorname{TV}(\mathbf p,\widehat{\mathbf p})<0.0364$, or $\|\mathbf p-\widehat{\mathbf p}\|_1<0.0728$.

\begin{figure}[t]
  \centering
  \includegraphics[width=\linewidth]{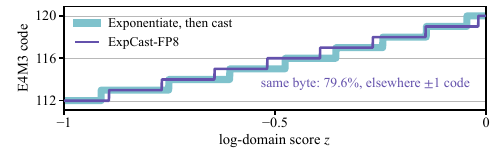}
  \caption{\textbf{ExpCast-FP8 writes the exponentiate-then-cast byte on \SameByteShare of scores.}
  The byte each path writes over one unit of $z$, which is one doubling of the probability.
  Where the two paths differ, they differ by one code.
  The pattern repeats in every unit with $z\geq-14$, where E4M3 stays normal.}
  \label{fig:expcast-bytes}
\end{figure}

\myparagraph{Subnormal entries.}
The ratios of \cref{eq:appendix-fp8-ratio} hold for normal E4M3 entries, which are the probabilities above $2^{-14}$ of the row maximum.
Let $\mathcal T=\{k:z_k<-14\}$ hold the remaining entries, and let $\tau=\max\{p(\mathcal T),\widehat p(\mathcal T)\}$.
Conditioning both rows on the normal entries and applying \cref{prop:direct-fp8-row} gives
\begin{equation}
  \operatorname{TV}(\mathbf p,\widehat{\mathbf p})<0.0364+\tau,
  \label{eq:appendix-fp8-tail}
\end{equation}
which is the 3.64\% of the main text when $\tau=0$, and therefore
\begin{equation}
  \|\widehat{\mathbf o}-\mathbf o\|_2
  \leq(0.0364+\tau)\operatorname{diam}_2(\mathbf V').
  \label{eq:appendix-fp8-output}
\end{equation}
This bound covers the probabilities only.
Quantizing the residual values adds a separate error term.

\myparagraph{Measured total variation.}
On 204.8K attention rows from 100 Wan2.2 heads, the total variation of ExpCast-FP8 averages \MeasuredTV, against \CastTV for an FP32 exponential followed by an E4M3 cast.
98.0\% of the rows stay below 3.64\%, and every row stays within $0.0364+\tau$, as \cref{eq:appendix-fp8-tail} requires.
We also measure 49.2K rows from 12 attention calls of Wan2.1 (the same DiT as the experts of Wan2.2-A14B), LongCat-Video, HunyuanVideo-1.5, and MiniMax-H3, with exact scores and only the probabilities quantized.
There, 98.5\% of the rows stay below 5\%, and every row again stays within $0.0364+\tau$.
The remaining rows put most of their mass below the E4M3 normal range, where ExpCast-FP8 is less accurate than an exact cast.
NVFP4 has no useful bound, since E2M1 rounds the smallest entries of each 16-element group to zero and thus $a=0$.
On the same rows, ExpCast-FP4 averages \FourBitTV, against \FourBitCastTV for an exact exponential followed by an NVFP4 cast.
In both formats, the rounding of the format itself thus accounts for most of the total variation.

\myparagraph{The denominator in practice.}
A kernel that divides by the sum of the FP32 probabilities but multiplies by the quantized ones returns $(\sum_k\widehat w_k/\sum_k w_k)\,\widehat{\mathbf p}^{\top}\mathbf V'$, a whole-row scale error that grows with $\|\mathbf o\|$ instead of $\operatorname{diam}_2(\mathbf V')$.
To measure this scale, we give all tokens the same value vector, so every output row should equal it.
On four captured layers of HunyuanVideo-1.5 and MiniMax-H3, SageAttention3 misses it on \SageConstVRows of the rows, with a mean ratio of 0.969--0.990, and at most 1.6\% of the rows are off by more than 5\%.
Our kernel returns every row exactly, with the exact exponential and with ExpCast-FP4.
On 44 captured calls, switching our 4-bit kernel without V-Smooth and ExpCast to the FP32 denominator lowers its per-call SNR by 1.31~dB.
With the FP32 denominator, this kernel and SageAttention3 with QK Hadamard both reach 26.06~dB on these calls.
On LongCat-Video with V-Smooth and ExpCast off, it lowers PSNR from 18.60 to 15.88~dB, by \DenDrop over the 32 paired prompts (95\% bootstrap interval 2.05--3.45~dB), and every prompt gets worse.
Without ExpCast, our 8-bit kernel can keep the FP32 row sum of FlashAttention-4, and the fidelity of the 8-bit V-Smooth rows in \cref{tab:main,tab:kernel-snr} was measured this way, with the same denominator as the 8-bit FlashAttention-4 baseline.
V-Smooth still leads that baseline with QK Hadamard by 0.5--3.8~dB of PSNR on every model and by 1.95~dB per call, and at 4 bits, where both kernels divide by the quantized sum, it adds 1.08~dB per call over ExpCast alone (\cref{tab:kernel-snr}).
The denominator and V-Smooth thus each contribute to the lead.

\myparagraph{The affine map of Direct-P zeroes a band of small probabilities.}
\phantomsection\label{sec:appendix-directp}
Direct-P~\cite{hu2026hardware} writes MXFP4 probabilities through a fitted affine map $u=\max(0,Ax+B)$ of the scaled score, followed by an E2M1 rounding.
We simulate its fast and accurate policies from the released code on the 49.2K rows above, with exact scores and only the probabilities quantized.
Their total variation is \DirectPTVFast and \DirectPTVAccurate, against \FourBitTV for ExpCast-FP4 and 7.9\% for an exact exponential followed by an MXFP4 cast.
The map rounds every value below about $2^{-3.2}$ of the block maximum to zero, while an exact cast keeps values down to about $2^{-4.6}$.
This band holds \DirectPZeroMass of the softmax mass (\cref{fig:directp-band}).

\myparagraph{Direct-P overflows without the online row maximum.}
The released MXFP4 policies of Direct-P skip the online rescaling.
The fast policy applies no shift, and the accurate policy shifts by the maximum of 32 keys sampled before the key loop.
On a Wan2.2 call at step 0, the row maximum of one head lies up to 273 octaves above this shift, and the exponentials overflow FP32 and the E8M0 scale.
Emulating the same MXFP4 probabilities against the true row maximum keeps every output finite, and our kernels keep the running maximum of online softmax.

\begin{figure}[t]
  \centering
  \includegraphics[width=\linewidth]{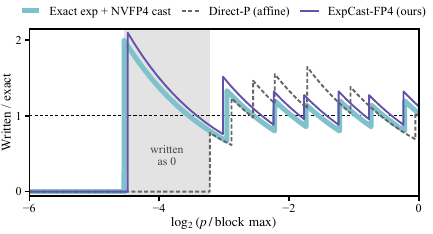}
  \caption{\textbf{The affine map of Direct-P writes a band of probabilities as zero.}
  Written value over exact probability for one block whose maximum is at 0: an exact exponential with an NVFP4 cast, ExpCast-FP4, and the fast policy of Direct-P (MXFP4, block of 32).
  Direct-P zeroes the shaded band, which the exact cast keeps.}
  \label{fig:directp-band}
\end{figure}

\begin{table*}[t]
  \centering
  \caption{\textbf{At a locked clock, the 8-bit FlashAttention-4 is slower than BF16, and \method stays 1.3\texttimes\ faster.} One attention
  call on the first $N = 37376$ of the 37440 tokens of the LongCat-Video layer-0 capture, $H = 32$, $d = 128$,
  on one B200, CUPTI kernel time. Peak is the 10th-percentile latency of 15
  launches started from an idle GPU. Steady is the median latency after twenty
  seconds at the 1000~W power limit, with the median SM clock of that window
  beside it. The 1350~MHz column is the median latency with the clock locked
  there. $^\dagger$:~training-free.}
  \label{tab:clock-regimes}
  {\setlength{\tabcolsep}{3.4pt}\footnotesize
\begin{tabularx}{\linewidth}{lc YYY Y YYY}
\toprule
 & & \multicolumn{3}{c}{Latency (ms, $\downarrow$)} & Steady & \multicolumn{3}{c}{Speedup over BF16 ($\uparrow$)} \\
\cmidrule(lr){3-5}\cmidrule(lr){7-9}
Kernel & QK/PV & Peak & Steady & 1350~MHz & clock (MHz) & Peak & Steady & 1350~MHz \\
\midrule
FlashAttention-4 & BF16 & 13.96 & 16.25 & 16.51 & 1485 & 1.00\texttimes & 1.00\texttimes & 1.00\texttimes \\
FlashAttention-4 (8-bit), FA4 exponential & 8/8 & 13.31 & 13.85 & 18.22 & 1890 & 1.05\texttimes & 1.17\texttimes & 0.91\texttimes \\
Attn-QAT$^\dagger$ & 4/8 & 11.74 & 12.09 & 16.07 & 1912 & 1.19\texttimes & 1.34\texttimes & 1.03\texttimes \\
\midrule
\ours \textbf{\method (V-Smooth + ExpCast)} & 8/8 & 9.44 & 10.94 & 12.69 & 1680 & 1.48\texttimes & 1.49\texttimes & \textbf{1.30}\texttimes \\
\ours \textbf{\method (ExpCast only)} & 8/8 & \textbf{9.42} & \textbf{10.80} & 12.93 & 1725 & 1.48\texttimes & \textbf{1.50}\texttimes & 1.28\texttimes \\
\bottomrule
\end{tabularx}}

\end{table*}

\begin{table}[t]
  \centering
  \caption{\textbf{Kernel-level fidelity on real MiniMax-H3 activations.}
  SNR (dB) and MSE of one attention call against its BF16 FlashAttention-4
  output, averaged over twenty captured calls (five DiT blocks \texttimes\ four
  denoising steps). Each row is one fixed kernel, and the V-Smooth rows group at every call. $^\star$:~deployed kernel. $^\dagger$:~training-free.}
  \label{tab:kernel-snr}
  {\setlength{\tabcolsep}{3pt}\footnotesize
\begin{tabularx}{\linewidth}{@{}Xccc@{}}
\toprule
Method & QK/PV & SNR$\uparrow$ & MSE$\downarrow$ \\
\midrule
\multicolumn{4}{@{}l}{\textit{8-bit, NVIDIA B200}} \\
SageAttention2 & 8/8 & 37.89 & 0.109 \\
FlashAttention-4 (8-bit) & 8/8 & 37.96 & 0.111 \\
\quad + QK Hadamard & 8/8 & 38.31 & 0.102 \\
Attn-QAT$^\dagger$ & 4/8 & 21.02 & 5.698 \\
\midrule
\ours \method (ExpCast only) & 8/8 & 38.17 & 0.107 \\
\ours \textbf{\method (V-Smooth)} & 8/8 & \textbf{40.26} & \textbf{0.068} \\
\ours \textbf{\method (V-Smooth + ExpCast)}$^\star$ & 8/8 & \underline{39.36} & \underline{0.085} \\
\midrule
\multicolumn{4}{@{}l}{\textit{4-bit, RTX PRO 6000}} \\
SageAttention3 & 4/4 & 23.08 & 2.922 \\
\quad + QK Hadamard & 4/4 & 23.77 & 2.520 \\
\midrule
\ours \method (ExpCast only) & 4/4 & 24.15 & 2.458 \\
\ours \textbf{\method (V-Smooth)} & 4/4 & \textbf{25.45} & \textbf{1.976} \\
\ours \textbf{\method (V-Smooth + ExpCast)}$^\star$ & 4/4 & \underline{25.23} & \underline{2.026} \\
\bottomrule
\end{tabularx}}

\end{table}

\section{Additional Results}
\label{sec:appendix-results}

\myparagraph{V-Smooth leads on single attention calls.}
\phantomsection\label{sec:appendix-kernel-snr}
\Cref{tab:kernel-snr} scores single attention calls on the post-RoPE operands of MiniMax-H3, captured on the BF16 FlashAttention-4 trajectory at DiT blocks 0, 12, 24, 36, and 48 and denoising steps 0, 1, 4, and 7.
With grouping at every call, V-Smooth + ExpCast leads the strongest baseline with QK Hadamard, FlashAttention-4 at 8 bits and SageAttention3 at 4 bits, by 1.05 and 1.46~dB of SNR.
Deployed, it groups only at steps 0 and 1, which matter most for the final video (\cref{fig:placement}), and leads by 0.39 and 0.84~dB.

\myparagraph{\method stays faster at a locked clock.}
\phantomsection\label{sec:appendix-clock}
\Cref{tab:clock-regimes} times the same kernels on one B200 at peak, in steady state, and with the clock locked at \LockedClockBTwoHundred, below the sustained clock of every kernel.
At the locked clock, only the work per cycle differs, and \method computes attention \LockedSpeedupBTwoHundred faster than BF16 FlashAttention-4.
The 8-bit FlashAttention-4 keeps the exponential of the BF16 kernel and is slower than BF16 (\LockedNaiveSpeedup).
Its 1.17\texttimes\ steady-state speedup comes from the higher clock that its lower power draw permits.

\myparagraph{The early steps also matter most at 8 bits.}
We repeat \cref{fig:placement} with the 8-bit kernel on LongCat-Video on one H200.
V-Smooth on the first quarter of the steps beats the same number of steps spread evenly by \PlacementFirstOverUniformEightBit~dB of PSNR.
V-Smooth on every step adds \PlacementAllOverFirstEightBit more at 4.6\% more time per clip.

\section{Visual Results}
\label{sec:appendix-qualitative}

\looseness=-1 Each figure below shows one clip generated three times from the same prompt and
seed: by FlashAttention-4 in BF16, by the baseline at matched precision, and by
\method, at the same five frame indices. PSNR is computed over all frames of that clip against its
FlashAttention-4 render. The first four results run on one B200 in the
deployed 8-bit configuration against SageAttention2, and the last four on one
RTX PRO 6000 in the deployed 4-bit configuration against SageAttention3.

\looseness=-1 The teaser (\cref{fig:teaser}) is a further MiniMax-H3 clip under the same 8-bit
configuration, cut at source frames 0, 60, 72, 192 and 216 of 243.
Its own PSNR is 16.54~dB with SageAttention2 and 18.96~dB with \method; \cref{fig:teaser} reports the means over 100 prompts.
Its full prompt is below.

{\small\itshape
\noindent\textup{\textbf{Description.}}
[Shot 1] Cinematic, action camera style, low-angle medium shot, tracking backward.
A young adult male parkour expert with a lean, athletic build and messy short brown hair dominates the center of the frame, wearing an oversized, loose-fitting white t-shirt and baggy light grey sweatpants.
He sprints vigorously down an urban marathon route on a sunny day.
Towering glass skyscrapers and a blurred crowd behind metal barricades form the background, illuminated by bright, harsh natural sunlight casting sharp shadows on the asphalt.
As the camera rapidly pulls back keeping pace with him, he approaches a blue folding hydration table, swiftly places both hands on the edge, and executes a fluid parkour vault over it.
He lands perfectly, his loose clothing violently fluttering in the wind, and resumes his sprint.
[Shot 2] At 00:04.500, the camera cuts to a side-profile medium wide shot, tracking right.
The man from Shot 1 blasts down the middle of the wide avenue, his arms pumping rhythmically.
Several exhausted marathon runners jog slowly in the blurred background, emphasizing his explosive speed.
He effortlessly hurdles over a bright orange traffic cone and a cluster of discarded white paper cups.
The wide-angle lens captures the intense flapping of his oversized shirt against his torso.
Sunlight catches the beads of sweat on his forehead as he maintains an unwavering, focused gaze straight ahead.
[Shot 3] At 00:09.200, the camera cuts to a front-facing close-up, pushing in.
The action instantly drops into dramatic slow motion.
A large overhead banner reading ``FINISH'' in bold yellow letters stretches across the upper background.
The runner leans his torso aggressively forward, his mouth slightly open, panting heavily.
He crashes through a bright red finish line tape, which stretches out and snaps cleanly across his chest.
The loose white fabric of his shirt ripples like water in the slowed time.
A warm sun flare streaks across the lens from the upper left as he throws both arms up into the air, his head tilting back in a victorious roar.

\noindent\textup{\textbf{Soundscape.}}
Rapid, heavy rubber-soled footsteps pound rhythmically against the asphalt, accompanied by the continuous, distinct rustling and snapping of loose fabric blowing in the wind.
A sharp, solid thud is heard as hands slap the plastic table during the vault.
In the background, a large crowd cheers continuously, layered with the high-pitched clatter of plastic cups rolling across the street.
The moment the slow motion begins, a deep, bass-heavy whoosh swallows the ambient noise, shifting the audio into muffled, stretched-out frequencies featuring a slow, echoing exhalation and the sharp, magnified tearing sound of the finish line tape snapping.

\noindent\textup{\textbf{Music.}}
High-energy, fast-paced electronic breakbeat music featuring a driving bassline and frantic drum loops, which dramatically filters out the high frequencies and drastically slows down in tempo during the final slow-motion sequence.
\par}

\newcommand{\qualstrip}[5]{%
  \begin{minipage}{\qualwidth}
    \centering\scriptsize
    \qualrow{FlashAttention-4 (BF16)}{\metric{PSNR}{reference}}
    \qualimg{#1_r0}\\[2.5pt]
    \qualrow{SageAttention2 (8-bit)}{\metric{PSNR}{#2\,dB}}
    \qualimg{#1_r1}\\[2.5pt]
    \qualrow{\best{\method}~\textbf{(8-bit)}}{\best{\metricb{PSNR}{#3\,dB}}}
    \qualimg{#1_r2}\\[1pt]
    \parbox{\qualwidth}{\scriptsize\itshape #4\par}
    \captionsetup{skip=3pt,hypcap=false}\captionof{figure}{#5}\label{fig:qual-#1}
  \end{minipage}%
}
\newcommand{\qualstripws}[5]{%
  \begin{minipage}{\qualwidth}
    \centering\scriptsize
    \qualrow{FlashAttention-4 (BF16)}{\metric{PSNR}{reference}}
    \qualimg{#1_r0}\\[2.5pt]
    \qualrow{SageAttention3 (4-bit)}{\metric{PSNR}{#2\,dB}}
    \qualimg{#1_r1}\\[2.5pt]
    \qualrow{\best{\method}~\textbf{(4-bit)}}{\best{\metricb{PSNR}{#3\,dB}}}
    \qualimg{#1_r2}\\[1pt]
    \parbox{\qualwidth}{\scriptsize\itshape #4\par}
    \captionsetup{skip=3pt,hypcap=false}\captionof{figure}{#5}\label{fig:qual-#1}
  \end{minipage}%
}

\begin{figure*}[p]
\centering
\qualstrip{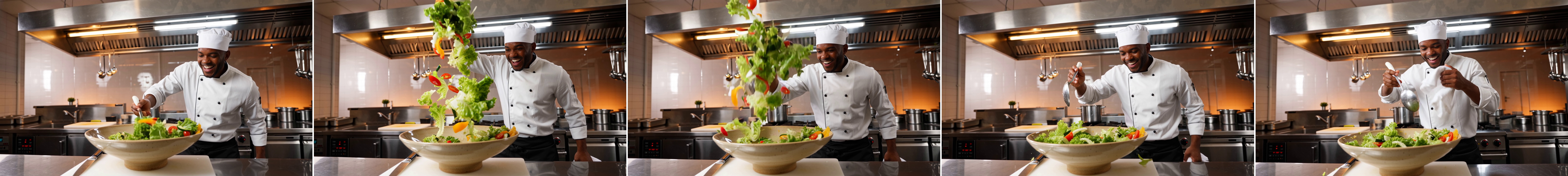}{10.88}{13.51}{%
  A dynamic scene captured in the style of a vibrant food photography shoot,
  showcasing a chef expertly tossing a salad in a large ceramic bowl. The chef,
  with a lively expression and focused intensity, moves with grace and
  precision, the salad spinning gracefully in the air before landing back in the
  bowl. \dots\ A mid-shot from a slightly elevated angle.}{8-bit results on Wan2.2-T2V-A14B, one B200.}

\vspace{18pt}
\qualstrip{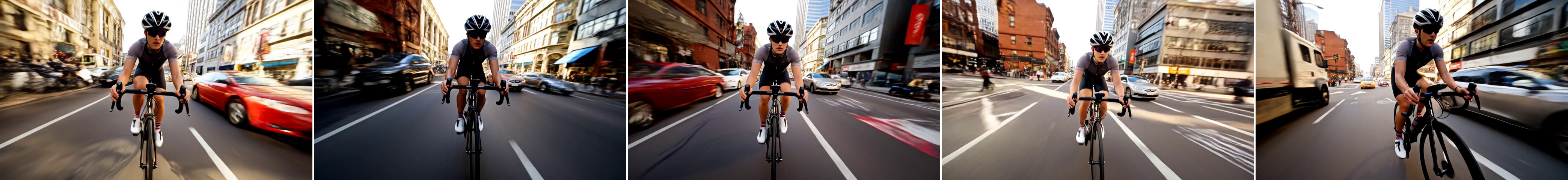}{14.96}{17.25}{%
  A dynamic first-person view of a cyclist navigating through a bustling city
  street, weaving skillfully between traffic and pedestrians. The cyclist is a
  young adult, wearing a helmet and a casual cycling jersey, pedaling
  energetically with a determined expression. \dots\ A close-up shot from a
  first-person perspective, emphasizing the cyclist's motion.}{8-bit results on LongCat-Video, one B200.}
\end{figure*}

\begin{figure*}[p]
\centering
\qualstrip{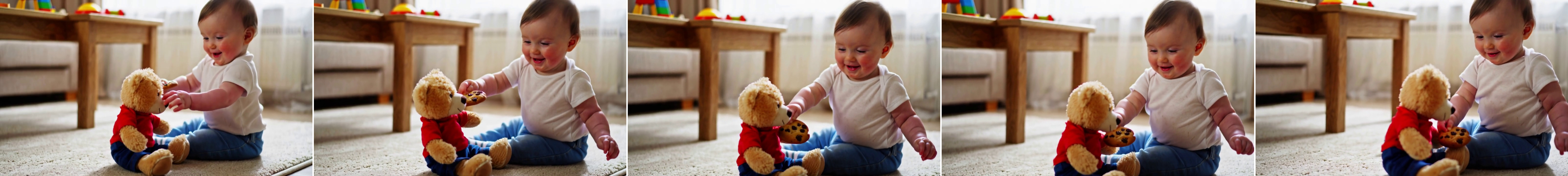}{15.02}{18.24}{%
  A charming photograph in a soft, warm lighting style, capturing a toddler
  sitting on a cozy carpet, happily sharing a chocolate chip cookie with a cute
  teddy bear. The toddler has rosy cheeks, big bright eyes, and a gentle smile,
  reaching out to offer the cookie to the bear, which leans in to accept it.
  \dots\ A medium shot with a slight angle.}{8-bit results on HunyuanVideo-1.5, one B200.}

\vspace{18pt}
\qualstrip{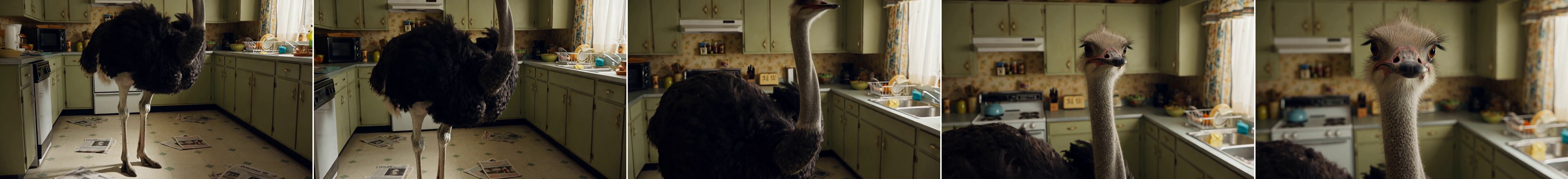}{19.53}{20.48}{%
  A slow cinematic push-in on an ostrich standing in a 1980s kitchen, the camera
  gradually zooming in to reveal the bird's curious expression. The kitchen is
  adorned with vintage appliances and Formica countertops, with a muted color
  palette of pastel greens and yellows. \dots\ The background features blurred
  details of old newspapers scattered on the floor and a faded floral wallpaper.
  \dots\ A close-up shot from a slightly lower angle.}{8-bit results on MiniMax-H3, one B200.}
\end{figure*}

\begin{figure*}[p]
\centering
\qualstripws{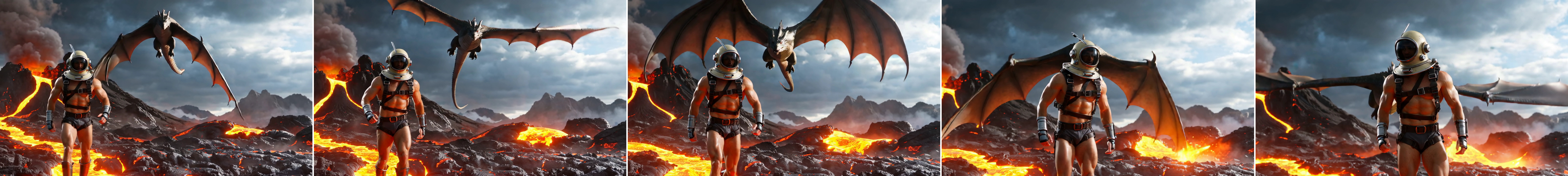}{12.07}{18.75}{%
  A macro shot in realistic style of a man wearing an antique diving helmet with
  dark glass and a jetpack, standing on a molten lava surface. He strides
  confidently, his body slightly bent forward, with a determined expression.
  Behind him, a majestic dragon soars through the sky \dots\ A medium shot with
  a slight tilt upwards, emphasizing both the man and the flying
  dragon.}{4-bit results on Wan2.2-T2V-A14B, one RTX PRO 6000.}

\vspace{18pt}
\qualstripws{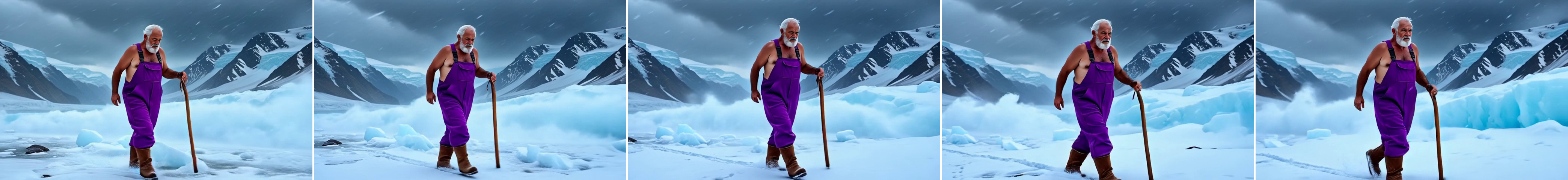}{14.29}{18.42}{%
  An old man, wearing vibrant purple overalls and sturdy cowboy boots, takes a
  leisurely stroll through the icy landscape of Antarctica during a fierce
  winter storm. \dots\ He carries a small, worn wooden walking stick, which he
  occasionally taps against the snow-covered ground. \dots\ A medium shot
  capturing the man in the midst of his stroll.}{4-bit results on LongCat-Video, one RTX PRO 6000.}
\end{figure*}

\begin{figure*}[p]
\centering
\qualstripws{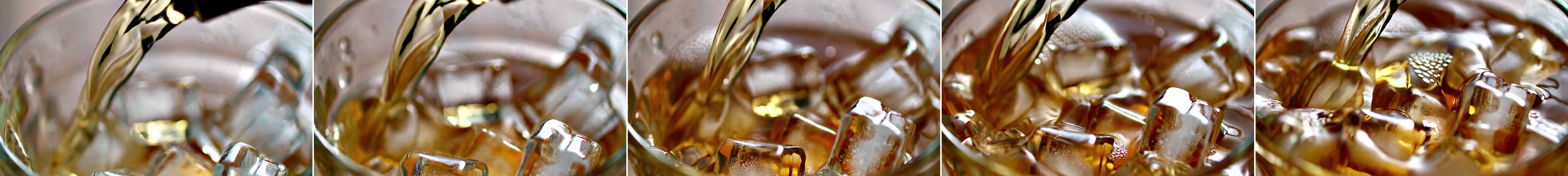}{10.80}{13.99}{%
  A close-up shot of a drink being poured over ice, showcasing the detailed flow
  of liquid interacting with the ice cubes. The drink cascades down, creating
  ripples and splashes on the surface of the ice, which glistens under the soft
  lighting. \dots\ A close-up from a slightly downward angle.}{4-bit results on HunyuanVideo-1.5, one RTX PRO 6000.}

\vspace{18pt}
\qualstripws{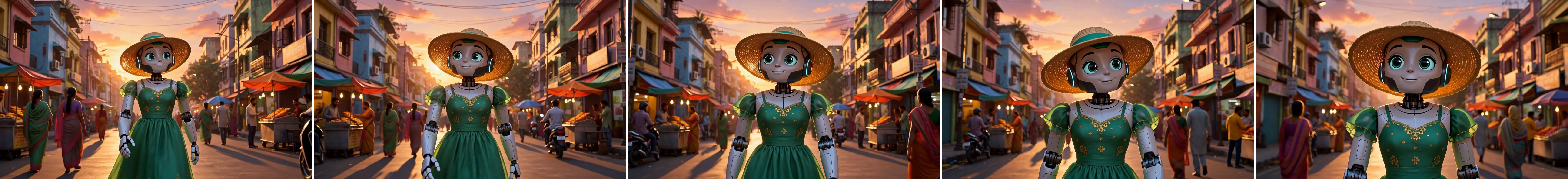}{15.11}{16.18}{%
  A vibrant illustration in the style of a modern comic book depicting a toy
  robot wearing a flowing green dress and a cheerful sun hat taking a pleasant
  stroll through the bustling streets of Mumbai, India, during a beautiful
  sunset. The robot has expressive eyes and a friendly smile, its arms swinging
  naturally as it walks. \dots\ A medium shot from a slightly elevated angle,
  capturing the robot's joyful movement.}{4-bit results on MiniMax-H3, one RTX PRO 6000.}
\end{figure*}

\end{document}